\documentclass[letter]{article}

\usepackage[table]{xcolor}         %
\usepackage{iclr2022_conference,times}
\iclrfinalcopy

\usepackage[american]{babel}

\usepackage[utf8]{inputenc} %
\usepackage[T1]{fontenc}    %
\usepackage[pdfauthor={Adri\'an Javaloy and Isabel Valera},
            pdftitle={RotoGrad: Gradient Homogenization in Multitask Learning},
            pdfsubject={Machine Learning},
            pdfkeywords={multitask learning, conflicting gradients, negative transfer}]{hyperref}       %
\usepackage{url}            %
\usepackage{booktabs}       %
\usepackage{amsfonts}       %
\usepackage{nicefrac}       %
\usepackage{microtype}      %

\usepackage{tikz}
\usepackage{tikz-cd}
\usepackage{standalone}
\usepackage{graphicx}
\usepackage{wrapfig}

\usepackage[page,toc]{appendix}
\usepackage{etoc}

\usepackage{fontawesome}
\usepackage{subcaption}
\usepackage{algorithm}
\usepackage{algorithmicx}
\usepackage{algpseudocode}
\usepackage{mathtools}
\usepackage{multirow,bigdelim}
\usepackage{makecell}
\usepackage[
    separate-uncertainty=true, 
    output-exponent-marker = \text{e},
    exponent-product={},
    binary-units=true
]{siunitx}

\usepackage{csquotes}

\usepackage{amsmath,amsfonts,bm}

\usepackage{xspace}

\newcommand{\captiona}{{\em (a)}\xspace}
\newcommand{\captionb}{{\em (b)}\xspace}

\def\Figref#1{Figure~\ref{#1}}

\def\eqref#1{equation~\ref{#1}}
\def\Eqref#1{Equation~\ref{#1}}

\def\1{\bm{1}}

\def\rvz{{\mathbf{z}}}

\def\vb{{\bm{b}}}

\def\vg{{\bm{g}}}

\def\vr{{\bm{r}}}

\def\vu{{\bm{u}}}
\def\vv{{\bm{v}}}

\def\vx{{\bm{x}}}
\def\vy{{\bm{y}}}
\def\vz{{\bm{z}}}

\def\mG{{\bm{G}}}

\def\mR{{\bm{R}}}

\def\mU{{\bm{U}}}
\def\mV{{\bm{V}}}
\def\mW{{\bm{W}}}
\def\mX{{\bm{X}}}
\def\mY{{\bm{Y}}}
\def\mZ{{\bm{Z}}}

\DeclareMathAlphabet{\mathsfit}{\encodingdefault}{\sfdefault}{m}{sl}
\SetMathAlphabet{\mathsfit}{bold}{\encodingdefault}{\sfdefault}{bx}{n}

\def\sR{{\mathbb{R}}}

\def\sY{{\mathbb{Y}}}

\newcommand{\R}{\mathbb{R}}

\usepackage{color}

\usepackage{amsfonts}
\usepackage{amsmath}
\usepackage{amsthm}
\usepackage{mathtools}

\usepackage{ifthen}

\renewcommand{\epsilon}{\varepsilon}

\DeclareMathOperator*{\argmin}{argmin}
\DeclareMathOperator*{\med}{med}
\DeclareMathOperator*{\std}{std}
\DeclareMathOperator*{\minimize}{minimize}

\DeclarePairedDelimiterX{\infdivx}[2]{(}{)}{%
	#1 \delimsize\|\, #2%
}

\newcommand{\E}[2][]{\Eoperator_{#1}\left[{#2}\right]}

\renewcommand{\P}{\mathbb{P}}

\newcommand{\seq}[3]{\left\{#1_{#2}\right\}_{\ifx&#3&\else #2=1\fi}^{#3}}  %
\newcommand{\range}[2][]{{\ifx&#1&1, 2, \dots, #2\else#1_1, #1_2,\dots, #1_{#2}\fi}}

\DeclarePairedDelimiter\abs{\lvert}{\rvert}%
\newcommand{\norm}[2][]{{||#2||}_{#1}}

\newtheorem{proposition}{Proposition}[section]

\theoremstyle{definition}
\newtheorem{definition}{Definition}[section]

\usepackage{amsthm}

\makeatletter
\renewenvironment{proof}[1][\proofname]{%
	\par\pushQED{\qed}\normalfont%
	\topsep6\p@\@plus6\p@\relax
	\trivlist\item[\hskip\labelsep\bfseries#1\@addpunct{.}]%
	\ignorespaces
}{%
	\popQED\endtrivlist\@endpefalse
}
\makeatother

\def\Tabref#1{Table~\ref{#1}}

\usepackage{pifont}

\newcommand{\blue}[1]{{\color{blue} #1}}

\usepackage{xspace}

\newcommand{\needcite}[1][]{\textbf{\color{red} [cite\ifthenelse{\equal{#1}{}}{}{~{#1}}]}\xspace}

\newcommand{\msg}[1]{\textbf{\color{green!50!black} [#1]}\newline}

\newcommand{\edit}[2][]{{\textcolor{orange}{\sout{#2}}}\ifthenelse{\equal{#1}{}}{}{\ifthenelse{\equal{#2}{}}{}{$\rightarrow$}{\textcolor{green!40!black!90}{#1}}}}

\newcommand{\backbone}{f}
\newcommand{\head}{h}

\newcommand{\ours}{RotoGrad\xspace}
\newcommand{\loss}{L}

\newcommand{\leader}{\mathcal{L}}
\newcommand{\follower}{\mathcal{F}}

\newcommand{\MTL}{MTL\xspace}
\newcommand{\ie}{that is\xspace}
\newcommand{\eg}{for example\xspace}
\newcommand{\thetab}{{\boldsymbol{\theta}}}
\newcommand{\phib}{{\boldsymbol{\phi}}}

\newcommand{\Roto}{{\operatorname{roto}}}
\newcommand{\Net}{\mathcal{N}}

\newcommand{\linear}[1]{{\small \texttt{[Dense-$#1$]}}}
\newcommand{\conv}[2]{{\small \texttt{[Conv-$#1$-$#2$]}}}

\newcommand{\batchnorm}{{\small \texttt{[BN]}}}

\newcommand{\relu}{{\small \texttt{[ReLU]}}}
\newcommand{\maxl}{{\small \texttt{[Max]}}}
\newcommand{\logsoftmax}{{\small \texttt{[Log-Softmax]}}}
\newcommand{\sigmoidl}{{\small \texttt{[Sigmoid]}}}

\usepackage{etoolbox}
\usepackage{threeparttable}
\usepackage[normalem]{ulem}
\robustify\bfseries
\robustify\uline
\robustify\tnote

\usepackage{soul}

\colorlet{darkgreen}{green!60!black!}
\hypersetup{citecolor=darkgreen}

\renewcommand{\E}[2][]{\operatorname{avg}_{#1} {#2}}

\usepackage[capitalize]{cleveref}

\newcommand{\maxf}[1]{{\cellcolor[gray]{0.87}} #1}

\makeatletter
\newcommand{\ssymbol}[1]{^{\@fnsymbol{#1}}}
\makeatother

\crefformat{section}{\S#2#1#3}
\crefformat{part}{\S#2#1#3}
\crefformat{chapter}{\S#2#1#3}
\crefformat{paragraph}{\P#2#1#3}
\crefformat{subparagraph}{\P#2#1#3}

\renewcommand{\msg}[1]{}
\renewcommand{\blue}[1]{}

\title{\ours: Gradient Homogenization in \\ Multitask Learning}

\author{%
  Adri\'an Javaloy\\
  Department of Computer Science\\
  Saarland University\\
  Saarbr\"ucken, Germany \\
  \href{mailto:ajavaloy@cs.uni-saarland.de?subject=[RotoGrad] Add subject here}{\texttt{ajavaloy@cs.uni-saarland.de}} \\
   \And
   Isabel Valera \\
   Department of Computer Science\\
  Saarland University\\
  Saarbr\"ucken, Germany \\
}

\begin{document}

    \etocdepthtag.toc{mtchapter}
    \etocsettagdepth{mtchapter}{subsection}
    \etocsettagdepth{mtappendix}{none}

    \maketitle

\begin{abstract}
  Multitask learning is being increasingly adopted in applications domains like computer vision and reinforcement learning. However, {optimally exploiting its advantages} remains a major challenge due to the effect of negative transfer.
  Previous works have tracked down this issue to the disparities in gradient magnitudes and directions across tasks when optimizing the shared network parameters. %
  While recent work has acknowledged that negative transfer is a two-fold problem, 
  existing approaches fall short. These methods only focus on either homogenizing the gradient magnitude across tasks;  or greedily change the gradient directions,  overlooking future conflicts.
  In this work, we introduce \ours, an algorithm that tackles  negative transfer as a whole: 
it jointly homogenizes gradient magnitudes and  directions,  while ensuring {training convergence}.
  We show that \ours outperforms competing methods in complex problems, including multi-label classification in CelebA and computer vision tasks in the NYUv2 dataset.
  A Pytorch implementation can be found in \url{https://github.com/adrianjav/rotograd}.
\end{abstract}

\section{Introduction} \label{sec:intro}

As neural network architectures get larger in order to solve increasingly more complex tasks, the idea of jointly learning multiple tasks (\eg, depth estimation and semantic segmentation {in computer vision}) with a single network {is becoming} more and more appealing. 
This is precisely the idea of multitask learning (MTL)~\citep{Caruana93multitasklearning}, which promises  higher performance in the individual tasks and  better generalization to unseen data, while drastically reducing the number of parameters~\citep{ruder2017overview}.

Unfortunately, sharing parameters between tasks may also lead to difficulties during training as tasks compete for shared resources, often resulting in poorer results than solving individual tasks, %
a phenomenon known as \emph{negative transfer}~\citep{ruder2017overview}. %
Previous works have tracked down this issue to the two types of differences between task gradients. %
First, \emph{differences in magnitude} across tasks can make some tasks dominate the others during the learning process.
Several methods have been proposed to homogenize gradient magnitudes such as MGDA-UB~\citep{sener2018multi}, GradNorm~\citep{chen2017gradnorm}, or IMTL-G~\cite{liu2021imtl}.
However, little attention has been put towards the second source of the problem: \emph{conflicting directions} of the gradients for different tasks. 
Due to the way gradients are added up, gradients of different tasks may cancel each other out if they point to opposite directions of the parameter space, thus leading to a poor update direction for a subset or even all tasks.
Only very recently a handful of works have started to propose methods to mitigate the conflicting gradients problem, \eg,  by removing conflicting parts of the gradients~\citep{yu2020gradient}, or randomly `dropping' some elements of the gradient vector~\citep{graddrop}.

In this work we propose \ours, an algorithm that tackles negative transfer as a whole by homogenizing both gradient magnitudes and directions across tasks. 
\ours addresses the gradient magnitude discrepancies by {re-weighting} task gradients at each step of the learning, {while encouraging learning those tasks that have converged the least thus far.}
In that way, it makes sure that no task is overlooked during training. 
Additionally, instead of  directly modifying gradient directions, \ours smoothly rotates the shared feature space differently for each task, seamlessly aligning gradients 
in the long run.
As shown by  our theoretical insights,  the cooperation between gradient magnitude- and direction-homogenization ensures the stability of the overall learning process. %
Finally, we run extensive experiments to empirically demonstrate that \ours leads to stable {(convergent)} learning, scales up to complex network architectures, and outperforms competing methods in multi-label classification settings in CIFAR10 and CelebA, as well as in computer vision tasks using the NYUv2 dataset.
{Moreover, we provide a simple-to-use library to include \ours in any Pytorch pipeline.}

\begin{figure*}[t]
	\centering
    \hfill %
    \begin{subfigure}[c]{.7\textwidth}
        \centering
        \includegraphics[width=\textwidth, keepaspectratio]{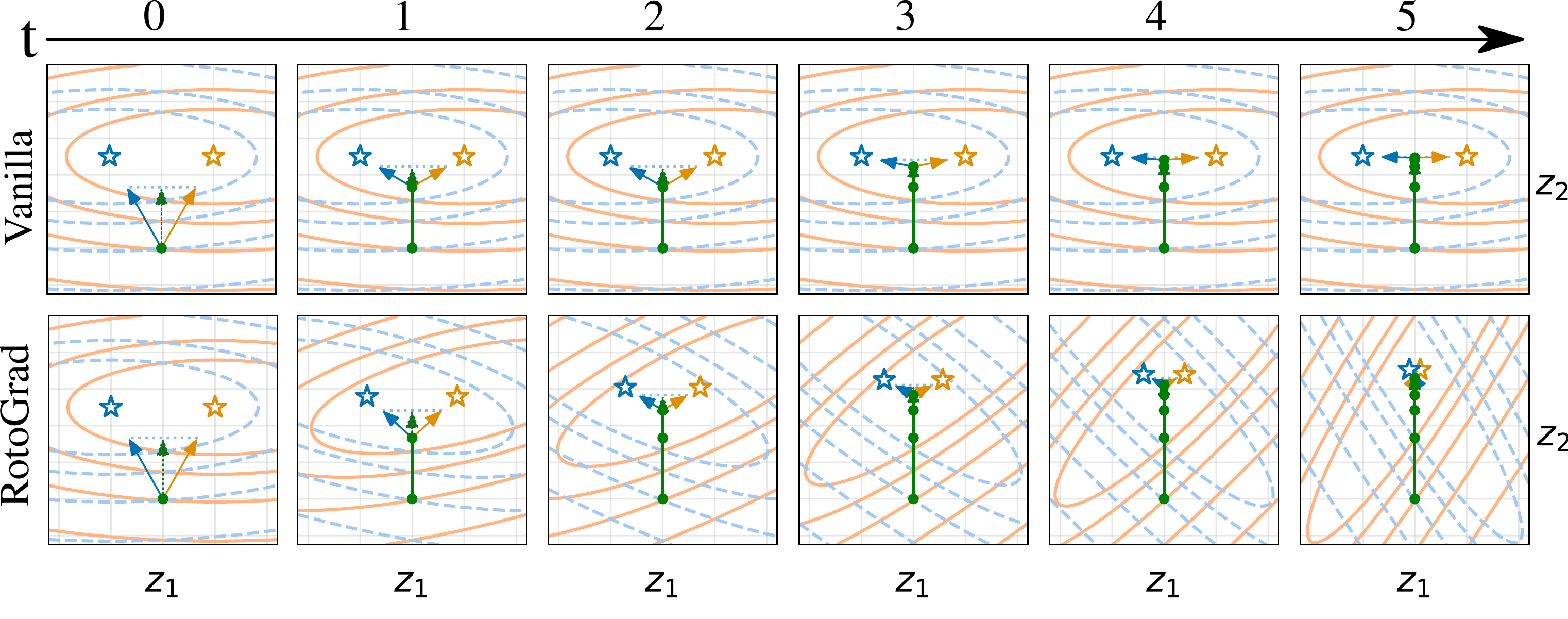}

        \caption{Convex avocado-shaped experiment.}\label{fig:toy1}
    \end{subfigure} %
    \hfill %
    \begin{subfigure}[c]{.285\textwidth}
        \centering
         \vspace{-0pt}
        \includegraphics[width=\textwidth, keepaspectratio]{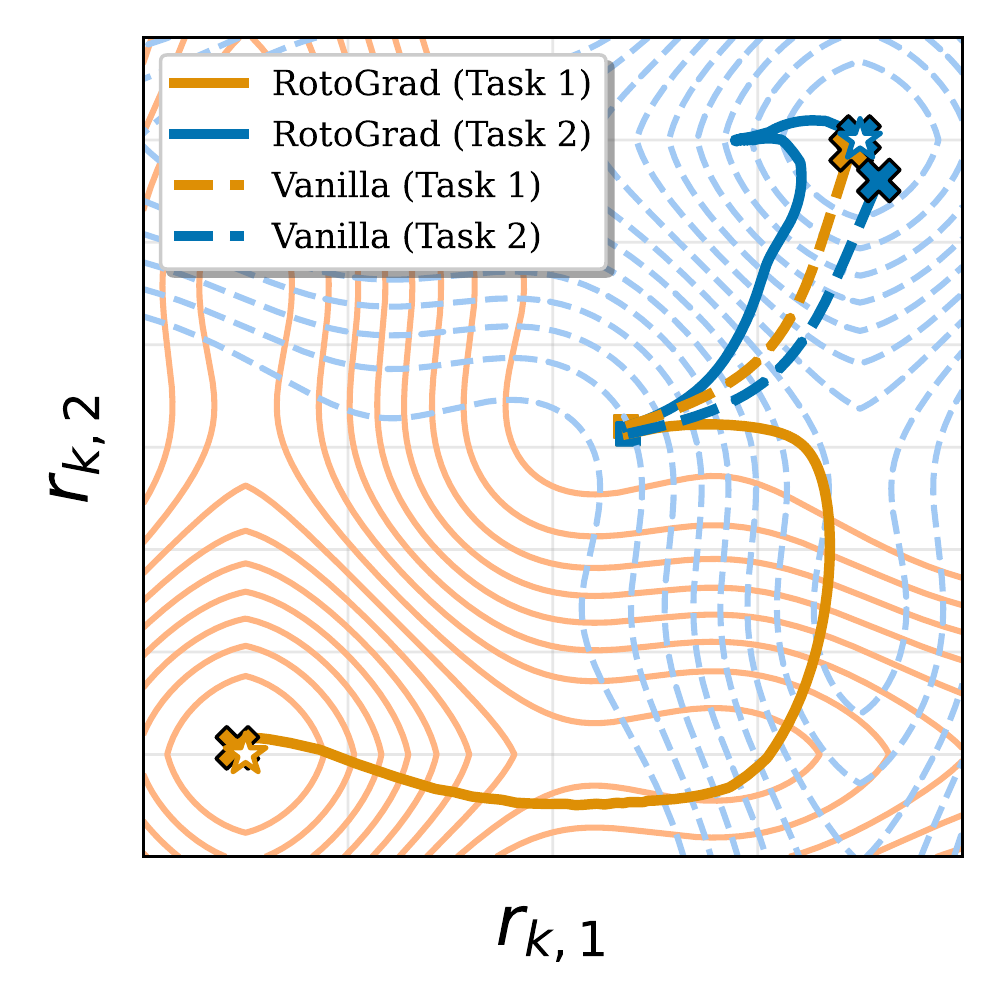}
        \vspace{-20pt}
        \caption{Non-convex experiment.}
        \label{fig:toy2}
    \end{subfigure} %
    \hfill
    \caption{{Level plots showing the evolution of two regression \MTL problems with/without \ours, see \Cref{sec:illustrative-examples}. \ours is able to reach the optimum (\faStarO) for both tasks. \captiona In the space of $\vz$, \ours rotates the function-spaces to align task gradients (\textcolor{blue}{blue}/\textcolor{orange!90!black}{orange} arrows), finding shared features $\vz$ (\textcolor{green!60!black}{green} arrow) closer to the (matched) optima. \captionb In the space of $\vr_k$, \ours rotates the shared feature $\vz$, providing per-task features $\vr_k$ that better fit each task.}}
    \label{fig:qualitative-plots}
\end{figure*} %

\section{Multitask learning and negative transfer} \label{sec:mtl}

\blue{
\begin{itemize}
	\item Introduce notation and architecture. Stress that the important part is to distinguish between shared and specific parameters.
	\item Introduce loss and gradient update. We are going to focus on feature-level gradients (due to the chain rule).
	\item Two-folded problem: magnitude and direction (two paragraph).
	\item But actually they are interleaved, changing the magnitude changes the directions and changing the directions changes the final magnitude (show figure and proposition).
\end{itemize}
}

\msg{Architecture and goal}
The goal of \MTL is to simultaneously learn $K$ different tasks, \ie, finding $K$ mappings from a common input dataset $\mX\in\R^{N\times D}$ to a task-specific set of labels $\mY_k\in\sY_k^N$.
Most settings consider a hard-parameter sharing architecture, which is characterized by two components: the \emph{backbone} and \emph{heads} networks. 
The backbone uses a set of shared parameters, $\thetab$, %
to transform each input $\vx\in\mX$ into a shared intermediate representation $\vz = \backbone(\vx; \thetab) \in \sR^d$, where $d$ is the dimensionality of $\vz$. 
Additionally, each task $k = \range{K}$ has a  head network~$\head_k$, with exclusive parameters~$\phib_k$, that takes this intermediate feature $\vz$ and outputs the prediction $\head_k(\vx) = \head_k(\vz; \phib_k)$ for the corresponding task.
This architecture is illustrated in \Figref{fig:extended-model}, where we have added task-specific rotation matrices $\mR_k$ that will be necessary for the proposed approach, \ours. Note that the general architecture described above is equivalent to the one in \Figref{fig:extended-model} when all rotations $\mR_k$ correspond to identity matrices, such that $\vr_k = \vz$ for all $k$. %

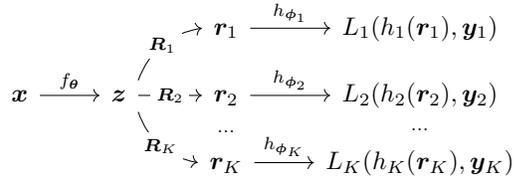
\begin{wrapfigure}[9]{r}{.5\textwidth}%
	\centering
	\vspace{-2em}
	    \begin{tikzcd}[row sep=small]
        & & \vr_1 \arrow[r, "\head_{\phib_1}"] & \loss_1(\head_1(\vr_1), \vy_1) \\ %
        \vx \arrow[r, "\backbone_{\thetab}"] & \vz \arrow[ur, start anchor=north east, end anchor=west, bend left, "\mR_1" description] \arrow[r, "\mR_2" description] \arrow[dr, start anchor=south east, end anchor=west, bend right, "\mR_K" description] & \vr_2 \arrow[d, draw=none, "\dots" description]\arrow[r, "\head_{\phib_2}"] & \loss_2(\head_2(\vr_2), \vy_2) \arrow[d, draw=none, "\dots" description]\\
        & & \vr_K \arrow[r, "\head_{\phib_K}"] & \loss_K(\head_K(\vr_K), \vy_K) \\ %
    \end{tikzcd}
	\vspace{-2em}
	\caption{Hard-parameter sharing architecture, including the rotation matrices $\mR_k$ of \ours.} %
	\label{fig:extended-model}
\end{wrapfigure} %

\msg{MOO optimization and linear scalarization, and negative transfer}
\MTL aims to learn the architecture parameters $\thetab,\range[\phib]{K}$ by simultaneously minimizing all task losses%
, \ie, $\loss_k(\head_k(\vx), \vy_k)$ for $k=1,\ldots, K$.
Although this is a priori a multi-objective optimization problem~\citep{sener2018multi}, in practice a single surrogate loss consisting of a linear combination of the task losses, $L = \sum_k \omega_k \loss_k$, is 
optimized. 
While this approach leads to a simpler optimization problem, it may also trigger \emph{negative transfer} between tasks, hurting the overall \MTL performance due to an imbalanced competition among tasks for the shared parameters~\citep{ruder2017overview}.

\msg{Gradient update, focus on feature-wise gradients, two-folded}
The negative transfer problem can be studied through the updates of the shared parameters~$\thetab$. %
At each training step, $\thetab$ is updated according to a linear combination of task gradients, \hbox{$\nabla_\thetab \loss = \sum_k \omega_k \nabla_\thetab \loss_k$}, which may suffer from two problems. %
First, \textbf{magnitude differences} of the gradients across tasks may lead to a subset of tasks dominating the total gradient, and therefore to the model prioritizing them over the others.
Second, \textbf{conflicting directions} of the gradients across tasks  %
may lead to update directions that do not improve any of the tasks. %
\Figref{fig:qualitative-plots} shows an example of poor direction updates (left) as well as magnitude dominance (right).

\msg{Magnitude. Differences in magnitude. Discuss they change direction}

\msg{Direction. Differences in direction. Discuss they change magnitude and cannot be left unsolved.}

\msg{objective and focus on $\vz$}
In this work, we tackle {negative transfer as a whole by}
homogenizing tasks gradients both in magnitude and direction.
To reduce overhead, we adopt the usual practice and homogenize gradients with respect to the shared feature $\vz$ (rather than $\thetab$), as all tasks share gradient up to that point, $\nabla_\thetab \loss_k = \nabla_\thetab \vz \cdot \nabla_\vz \loss_k$.
Thus, from now on we focus on %
feature-level task gradients $\nabla_\vz \loss_k$.

\section{\ours}

\msg{reminder of what we provide}
In this section we introduce \ours, a novel algorithm that addresses the negative transfer problem as a whole. %
{\ours consists of two building blocks which, respectively, homogenize task-gradient magnitudes and directions. Moreover, these blocks complement each other and provide convergence guarantees of the network training.}
Next, we detail each of these building blocks and show how they are combined towards an effective \MTL{} learning process. 

\blue{
\begin{itemize}
	\item Our goal here, similar convergence rate (not leave anyone behind).
	\item Explain algorithm.
	\item Parameter less, it has a similar effect as normalized gradient descent in single-task learning.
	\item I know it will converge if all point to similar directions, otherwise I cannot ensure it.
\end{itemize}
}

\msg{dominance reminder. Goal equal magnitude + tasks converging at similar rates.}
\subsection{Gradient-magnitude homogenization}\label{sec:rotograd-i}
As discussed in \Cref{sec:mtl}, we aim to homogenize gradient magnitudes across tasks, as large magnitude disparities can lead to a subset of tasks dominating the learning process. Thus, the first goal of \ours is to homogenize the magnitude of the gradients across tasks at each step of the training. 

\msg{show how it is easy to make them have the same magnitude, 1 dof}
Let us denote the feature-level task gradient of the $k$-th task for the $n$-th datapoint, at iteration $t$, by \hbox{$\vg_{n,k} \coloneqq \nabla_\vz \loss_k(\head_k(\vx_n), \vy_{n,k})$}, and its batch versions by $\mG_k^\top \coloneqq [\vg_{1,k}, \vg_{2,k}, \dots, \vg_{B,k}]$, where $B$ is the batch size. 
Then, equalizing gradient magnitudes amounts to finding weights $\omega_k$ that normalize and scale each gradient $\mG_k$, \ie,
\begin{equation}
    \norm{\omega_k\mG_k} = \norm{\omega_i\mG_i} \quad \forall i \iff \omega_k \mG_k = \frac{C}{\norm{\mG_k}} \mG_k = C \mU_k \quad \forall k, \label{eq:homogenize-magnitudes}
\end{equation}
where $\mU_k \coloneqq \frac{\mG_k}{\norm{\mG_k}} $ denotes the normalized task gradient and $C$ is the target magnitude for all tasks. 
\msg{degree of freedom, asumption}
Note that, in the above expression, $C$ is  a free parameter that we need to select. 

In \ours, we select $C$ such that  all tasks converge at a similar rate. %
We motivate this choice by the fact that, by scaling all gradients, we change their individual step size, interfering with the convergence guarantees provided by their Lipschitz-smoothness (for an introduction to optimization see, \eg, \citep{nesterov2018lectures}).
Therefore, we seek for the value of $C$ providing the best step-size for those tasks that have converged the least up to iteration $t$.
\msg{similar rate of (relative) convergence -> weighted norm average}
Specifically, we set $C$ to be a convex combination of the task-wise gradient magnitudes, \hbox{$C \coloneqq \sum_k \alpha_k \norm{\mG_k}$}, where the weights $\range[\alpha]{K}$ measure the relative convergence of each task and sum up to one, \ie,
\begin{equation}
    \alpha_k = \frac{\norm{\mG_k}/\norm{\mG_k^0}}{\sum_i \norm{\mG_i} / \norm{\mG_i^0}}, %
\end{equation}
with $\mG_k^0$ being the initial gradient of the $k$-th task, i.e., the gradient at iteration $t=0$ of the training.

\msg{pros: parameter-less, hyperparameter-less, regularizing effect to escape from saddle points}
As a result, we obtain a (hyper)parameter-free approach that equalizes the gradient magnitude across tasks to encourage learning slow-converging tasks. %
Note that the resulting approach resembles
Normalized Gradient Descent (NGD)~\citep{cortes2006finite} for single-task learning, which has been proved to quickly escape saddle points during optimization~\citep{murray2017revisitingngd}.
Thus, we expect a similar behavior for \ours, where slow-converging tasks will force quick-converging tasks to escape from saddle points.%

\msg{cons: it may diverge in the presence of conflicting gradients, but we will solve it}
Whilst the algorithm works well in general, its simplicity also facilitates unfavorable settings. For example, in the presence of noisy tasks that do not progress; or in scenarios where, when one task improves, there is always another task that deteriorates.
In \cref{app:proofs} we show that, when gradients do not conflict in direction with each other (which we pursue next), following the gradient $C\sum_k\mU_k$ improves all task losses for the given batch. 
This result, while simple, provides insights in favor of having as \textit{desideratum} of an efficient \MTL{} pipeline the absence of conflicting gradients.

\subsection{Gradient-direction homogenization}
\label{sec:rotograd-ii}

\blue{
\begin{itemize}
	\item Introduce intuition, idea.
	\item This is the most novel thing: we change the gradient directions over time through rotations.
	\item 
\end{itemize}
}

\msg{Intro, what we provide.}
In the previous subsection, we have shown that avoiding conflicting gradients may not only be necessary to avoid negative transfer, but also to ensure the stability of the training. 
In this section we introduce the second building block of \ours, an algorithm that homogenizes task-gradient directions. 
The main idea of this approach is to 
smoothly rotate the feature-space $\vz$ in order to 
reduce the gradient conflict between tasks---in following iterations---of the training by  bringing 
(local) optima for different tasks closer to each other (in the parameter space). 
 As a result, it complements the previous magnitude-scaling approach and reduces the likelihood of the training to diverge.

\msg{Intuition}
In order to homogenize gradients, for each task $k=1, \ldots, K$, \ours introduces a matrix $\mR_k$ so that, instead of optimizing $\loss_k(\vz)$ with $\vz$ being the last shared representation, we optimize an equivalent (in optimization terms, as it is a bijective mapping) loss function $\loss_k(\mR_k\vz)$. 
As we are only interested in changing directions (not the gradient magnitudes), we choose $\mR_k \in SO(d)$ to be a rotation matrix\footnote{The special orthogonal group, $SO(d)$, denotes the set of all (proper) rotation matrices of dimension $d$.} leading to per-task representations $ \vr_k \coloneqq \mR_k\vz$. 
\ours thus extends the standard MTL architecture by adding task-specific rotations before each head, as depicted in \Figref{fig:extended-model}.

\msg{pose the actual problem}
Unlike all other network parameters, matrices $\mR_k$ do not seek to reduce their task's loss. 
Instead, these additional parameters are optimized to reduce the direction conflict of the gradients across tasks.
To this end, for each task we optimize $\mR_k$ to maximize the batch-wise cosine similarity or, equivalently, to minimize
\begin{equation}
    \leader^k_{\operatorname{rot}} \coloneqq - \sum_n \langle \mR_k^\top \, \widetilde{\vg}_{n,k}, \vv_n \rangle, \label{eq:roto-loss}
\end{equation}
where $\widetilde{\vg}_{n,k} \coloneqq \nabla_{\vr_k} \loss_k(\head_k(\vx_n), \vy_{n,k}))$ (which holds that $\vg_{n,k} = \mR_k^\top\widetilde{\vg}_{n,k}$) and $\vv_n$ is the target vector that we want all task gradients pointing towards. 
We set the  target vector $\vv_n$ to be the gradient we would have followed if all task gradients weighted the same, \ie, $\vv_n \coloneqq \frac{1}{K} \sum_k \vu_{n,k}$, where $\vu_{n,k}$ is {a row vector} of the normalized batch gradient matrix $\mU_k$, as defined before.
\msg{joint bi-level problem}
As a result, in each training step of \ours we simultaneously optimize the following two problems:
\begin{equation}
\Net\text{{\small etwork}:}~\minimize_{\thetab, \{\phib\}_k}~\sum_k \omega_k\,\loss_k. \label{eq:stackelberg-rotograd}, \qquad 
\mathcal{R}\text{{\small otation}:}~\minimize_{\{\mR_k\}_k}~\sum_k \leader^k_{\operatorname{rot}}
\end{equation}%
\msg{how we solve this issue}
The above problem can be interpreted as a Stackelberg game: 
 a two player-game in which \emph{leader} and \emph{follower} alternately make moves in order to minimize their respective losses, $\loss_l$ and $\loss_f$, and the leader knows what will be the follower's response to their moves.
Such an interpretation allows us to derive simple guidelines  to guarantee {training convergence}---\ie, that the network loss does not oscillate as a result of optimizing the two different objectives in \Eqref{eq:stackelberg-rotograd}. 
Specifically, following \citet{fiez2019convergence}, we can ensure that 
{problem~\ref{eq:stackelberg-rotograd} converges}
as long as the rotations' optimizer (leader) is a slow-learner compared with the network optimizer (follower).
That is, as long as we make the rotations' learning rate decrease faster than that of the network, we know that \ours will converge to a local optimum for both objectives. 
A more extensive discussion can be found in \Cref{app:stackelberg}.

\subsection{\ours: the full picture} \label{sec:rotograd-iii}
After the two main building blocks of \ours, we can now summarize the overall proposed approach in Algorithm~\ref{alg:rotograd}.
At each step, \ours first homogenizes the gradient magnitudes such that there is no dominant task and the step size is set by the slow-converging tasks. 
{Additionally, \ours smoothly updates the rotation matrices---using the local information given by the task gradients---to seamlessly align task gradients in the following steps, thus reducing direction conflicts.}

\begin{algorithm}
    \caption{Training step with \ours.}
    \label{alg:rotograd}
    \textbf{Input} input samples $\mX$, task labels $\{\mY_k\}$, network's (\ours's) learning rate $\eta$ ($\eta_\Roto$) \\
    \textbf{Output} backbone (heads) parameters $\thetab$ ($\{\phib_k\}$), \ours's parameters $\{\mR_k\}$ 
    \begin{algorithmic}[1]
        \State compute shared feature $\mZ = \backbone(\mX; \thetab)$
        \For {$k = \range{K}$}
            \State compute task-specific loss $\loss_k = \sum_n \loss_k(\head_k(\mR_k\vz_{n}; \phib_k), \vy_{n,k})$
            \State compute gradient of shared feature $\mG_k = \nabla_{\vz} \loss_k$
            \State compute gradient of task-specific feature $\widetilde{\mG}_k = \mR_k\mG_k$ \Comment{Treated as constant w.r.t. $\mR_k$.}
            \State compute unitary gradients $\mU_k = \mG_k / \norm{\mG_k}$
            \State compute relative task convergence $\alpha_k = \norm{\mG_k}/\norm{\mG_k^0}$
        \EndFor
        \State make $\{\alpha_k\}$ sum up to one $[\range[\alpha]{K}] = {[\range[\alpha]{K}]}/{\sum_k \alpha_k}$
        \State compute shared magnitude $C = \sum_k \alpha_k \norm{\mG_k}$
        \State update backbone parameters $\thetab = \thetab - \eta C \sum_k \mU_k$
        \State compute target vector $\mV = \frac{1}{K} \sum_k \mU_k$
        \For {$k = \range{K}$}
            \State compute \ours's loss $\loss_k^{\Roto} = -\sum_n \langle \mR_k^\top \widetilde{\vg}_{n,k}, \vv_n \rangle$
            \State update \ours's parameters $\mR_k = \mR_k - \eta_\Roto \nabla_{\mR_k} \loss_k^\Roto$
            \State update head's parameters $\phib_k = \phib_k - \eta \nabla_{\phib_k} \loss_k$
        \EndFor
    \end{algorithmic}
\end{algorithm}

\subsection{Practical considerations} \label{sec:scaling}

In this section, we discuss the main practical considerations to account for when implementing \ours and propose efficient solutions. 

\msg{Unconstrained optimization}
\textbf{Unconstrained optimization.} 
As previously discussed, parameters $\mR_k$ are defined as rotation matrices, and thus the \emph{Rotation} optimization in problem~\ref{eq:stackelberg-rotograd} is a constrained problem.
While this would typically imply using expensive algorithms like Riemannian gradient descent~\citep{absil2009optimization}, we can leverage recent work on manifold parametrization~\citep{lezcano2019cheap} and, instead, apply unconstrained optimization methods by automatically\footnote{For example, Geotorch~\citep{lezcano2019trivializations} makes this transparent to the user.} parametrizing $\mR_k$ via exponential maps on the Lie algebra of $SO(d)$.

\msg{Memory complexity and time complexity}
\textbf{Memory efficiency and time complexity.}
As we need one rotation matrix per task, we have to store $O(Kd^2)$ additional parameters.
In practice, we only need ${Kd(d-1)}/{2}$ parameters due to the aforementioned parametrization and, in most cases, this amounts to a small part of the total number of parameters.
Moreover, as described by~\citet{lezcano2019cheap}, parametrizing $\mR_k$ enables efficient computations compared with traditional methods, with a time complexity of $O(d^3)$ independently of the batch size.
In our case, the time complexity is of $O(Kd^3)$, which scales better with respect to the number of tasks than existing methods (\eg, $O(K^2d)$ for PCGrad~\citep{yu2020gradient}).
Forward-pass caching and GPU parallelization 
can further reduce training time.

\msg{Only a subspace, block-diagonal rotation matrices}
\textbf{Scaling-up \ours.}
Despite being able to efficiently compute and optimize the rotation matrix $\mR_k$, in %
domains like computer vision, where
the size $d$ of the shared representation $\vz$ is large, the time complexity for updating the rotation matrix may become comparable to the one of the network updates. 
In those cases, we propose to only rotate a subspace of the feature space, \ie, rotate only $m << d$ dimensions of $\vz$.
Then, we can simply apply a transformation of the form $\vr_k = [\mR_k\vz_{1:m}, \vz_{m+1:d}]$, where $\vz_{a:b}$ denotes the elements of $\vz$ with indexes $a, a+1, \dots, b$. 
While there exist other possible solutions, such as using block-diagonal rotation matrices $\mR_k$, we defer them to future work.

\section{Illustrative examples} \label{sec:illustrative-examples}

\msg{What we do, appendix}
In this section, we illustrate the behavior of \ours in {two}  synthetic scenarios, providing clean qualitative results about its effect on the optimization process.
\Cref{app:setups} provides  a detailed description of the experimental setups.

\msg{Symmetric multitask regression problems, introduce}
To this end, we propose two different multitask regression problems of the form
\begin{equation}
    \loss(\vx) = \loss_1(\vx) + \loss_2(\vx) =  \varphi(\mR_1\backbone(\vx; \thetab), 0) + \varphi(\mR_2\backbone(\vx; \thetab), 1), \label{eq:illustrative}
\end{equation}
where $\varphi$ is a test 
function with a single global optimum whose position is parametrized by the second argument, \ie, both tasks are identical (and thus related) up to a translation.
We use a single input $\vx\in\sR^2$ and drop task-specific network parameters.
As backbone, we take a simple network of the form $\vz = \mW_2 \max(\mW_1\vx + \vb_1, 0) + \vb_2$ with $\vb_1\in\sR^{10}, \vb_2\in\sR^2$, and $\mW_1, \mW_2^\top \in \sR^{10\times2}$.

\msg{Describe $\varphi$ and explain plots}
For the first experiment we choose a simple (avocado-shaped) convex objective function and, for the second one, we opt for a non-convex function with several local optima and a single global optimum.
Figure~\ref{fig:qualitative-plots} shows the training trajectories in the presence (and absence) of \ours in both experiments, depicted as level plots in the space of $\vz$ and $\vr_k$, respectively. %
\msg{explain both results individually}
We can observe that in the first experiment~(\Figref{fig:toy1}), \ours finds both optima %
by rotating the feature space and matching the (unique) local optima of the tasks.
{Similarly, the second experiment (\Figref{fig:toy2}) shows that, as we have two symmetric tasks and a non-equidistant starting point, in the vanilla case the optimization is dominated by the task with an optimum closest to the starting point. \ours avoids this behavior by equalizing gradients and, by aligning gradients, is able to find the optima of both functions.}

\section{Related Work}

\blue{
\begin{itemize}
	\item Two types of solutions, architecture and gradient based solutions.
	\item The architecture-based we focus on hard parameter sharing for simplicity, although it could be applied to any architecture with clear shared-specific parameters such as mtan.
	\item We also leave aside solutions based on task-grouping, since they are complementary and orthogonal to this work.
	\item Regarding gradient homogenization solutions there are two clear types: task-weighting and the others. Briefly explain what they consist.
\end{itemize}
}

Understanding and improving the interaction between tasks is one of the most fundamental problems of MTL, since any improvement in this regard would translate to all MTL systems. 
Consequently, several approaches to address this problem have been adopted in the literature. %
Among the different lines of work, the one most related to the present work is gradient homogenization.

\textbf{Gradient homogenization.} 
Since the problem is two-fold, there are two main lines of work.
On the one hand, we have task-weighting approaches that focus on alleviating magnitude differences.
Similar to us, GradNorm~\citep{chen2017gradnorm} attempts to learn all tasks at a similar rate, yet they propose to learn these weights as parameters. 
Instead, we provide a closed-form solution in \Eqref{eq:homogenize-magnitudes}, and so does IMTL-G~\cite{liu2021imtl}. 
However, IMTL-G scales all task gradients such that all projections of $\mG$ onto $\mG_k$ are equal.
MGDA-UB~\citep{sener2018multi}, instead, adopts an iterative method based on the Frank-Wolfe algorithm in order to find the set of weights $\{\omega_k\}$ (with $\sum_k\omega_k=1$) such that $\sum_k \omega_k \mG_k$ has minimum norm.
On the other hand, recent works have started to put attention on the conflicting direction problem. 
\citet{maninis2019attentive} and \citet{sinha2018gradient} proposed to make task gradients statistically indistinguishable via adversarial training.
More recently, PCGrad~\citep{yu2020gradient} proposed to drop the projection of one task gradient onto another if they are in conflict,
whereas GradDrop~\citep{graddrop} randomly drops elements of the task gradients based on a sign-purity score.
Contemporaneously to this work, improved versions of MGDA~\citep{desideri2012multiple} and PCGrad have been proposed by \citet{liu2021conflictaverse} and \citet{wang2021gradient}, respectively.

The literature also includes approaches which, while orthogonal to the gradient homogenization, are \textbf{complementary to our work} and thus could  be used along with \ours. Next, we provide a brief overview of them. 
A prominent approach for MTL is {task clustering}{, \ie, selecting which tasks should be learned together}.  
This approach dates back to the original task-clustering algorithm~\citep{thrun1996discovering}, but new work in this direction keeps coming out \citep{standley2019tasks,zamir2018taskonomy,shen2021variational,fifty2021efficiently}.
Alternative approaches, \eg, scale the loss of each task differently based on different criteria such as task uncertainty~\citep{kendall2018multi}, task prioritization~\citep{guo2018dynamic}, or similar loss magnitudes~\citep{liu2021imtl}.
Moreover, while most models fall into the hard-parameter sharing umbrella, there exists other architectures in the literature.
Soft-parameter sharing architectures~\citep{ruder2017overview}, \eg, do not have shared parameters but instead impose some kind of shared restrictions to the entire set of parameters.
An interesting approach consists in letting the model itself learn which parts of the architecture should be used for each of the tasks~\citep{guo2020learning,misra-stitch,sun2020adashare,vandenhende2019branched}.
Other architectures, such as MTAN~\citep{DBLP:mtan}, make use of task-specific attention to select relevant features for each task.
Finally, %
similar issues %
have also been studied in other domains like meta-learning~\citep{flennerhag2019meta} and continual learning~\citep{lopez2017gradient}.

\section{Experiments} \label{sec:experiments}

\blue{
\begin{itemize}
    \item Synthetic experiments, we visualize and test \ours in extreme scenarios 
    \item Stability results, we show the effect of a slow-learner in the stability of the training dynamics
    \item TODO
\end{itemize}
}

\msg{Explain what we are going to show/answer}
In this section we assess the performance of \ours on a wide range of datasets and \MTL{} architectures. 
First, we check the effect of the learning rates of the rotation and network updates on the stability of \ours.
Then, with the goal of applying \ours to scenarios with large sizes of $\vz$, we explore the effect of rotating only a subspace of $\vz$. %
Finally, we compare our approach with competing \MTL{}  solutions in the literature, showing that \ours consistently outperforms all existing  methods.
Refer to \Cref{app:experiments} for a more details on the experiments and additional results.

\textbf{Relative task improvement.}
Since MTL uses different metrics for different tasks, throughout this section we group results by means of the relative task improvement, first introduced by~\citet{maninis2019attentive}.
Given a task $k$, and the metrics obtained during test time by a model, $M_k$, and by a baseline model, $S_k$, which consists of $K$ networks trained on each task individually, the relative task improvement for the $k$-th task is defined as
\begin{equation}
    \Delta_k \coloneqq 100 \cdot (-1)^{l_k} \frac{M_k - S_k}{S_k},
\end{equation}
where $l_k=1$ if $M_k < S_k$ means that our model performs better than the baseline in the $k$-th task, and $l_k=0$ otherwise.
We depict our results using different statistics of $\Delta_k$ such as its mean~($\E[k]{\Delta_k}$), maximum ($\max_k \Delta_k$), and median ($\med_k\Delta_k$) across tasks.

\textbf{Statistical significance.} We highlight  significant improvements according to a one-sided paired t-test ($\alpha = 0.05$), with respect to MTL with vanilla optimization (marked with $\dagger$ in each table).

\subsection{Training stability} \label{sec:stability}

\begin{wrapfigure}[12]{r}{.4\textwidth}%
    \centering
    \vspace{-3em}
    \includegraphics[keepaspectratio, width=.4\columnwidth]{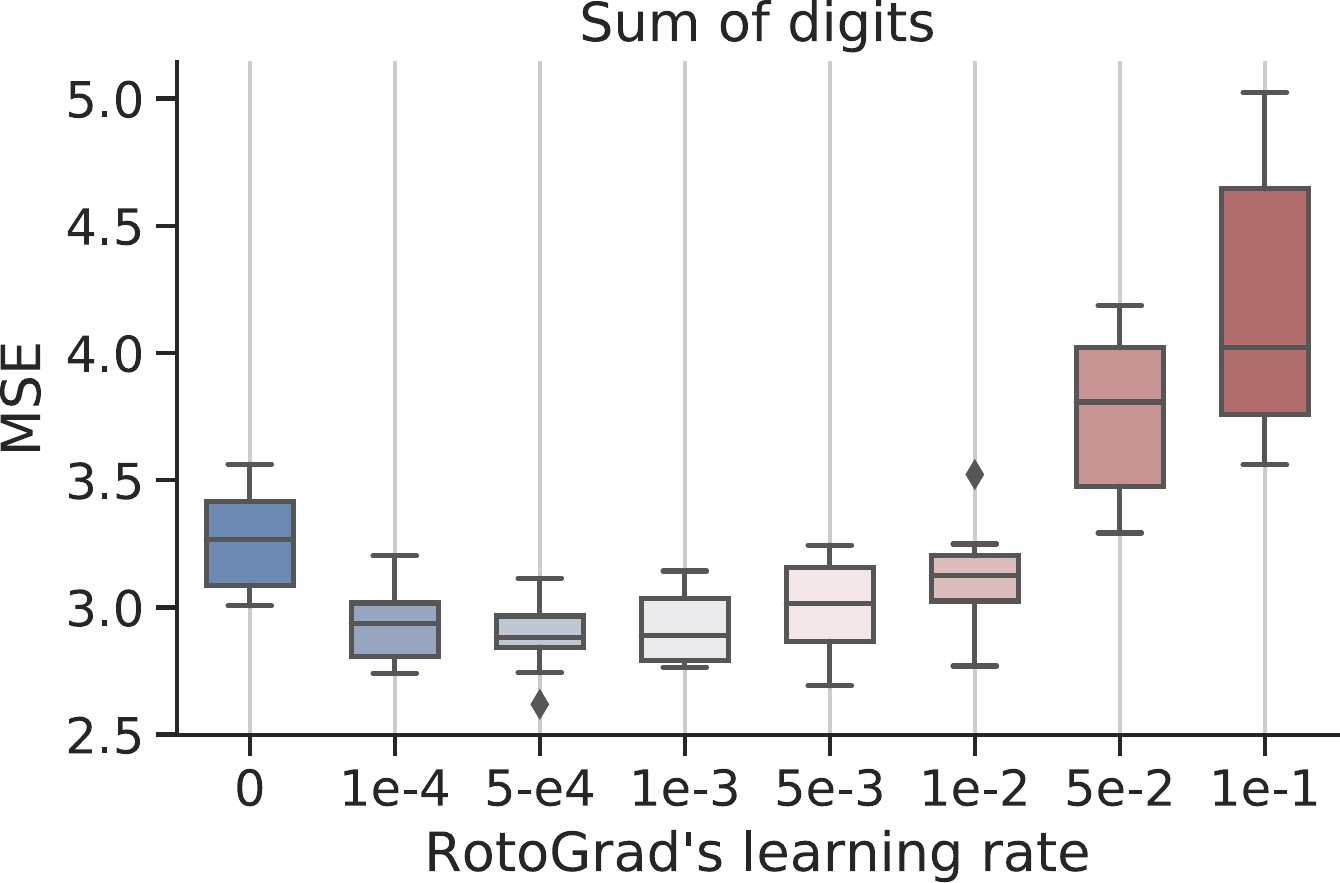}
    \caption{Test error on the sum of digits task for different values of \ours's learning rate on multi-MNIST. 
    }
    \label{fig:learning_rate}
\end{wrapfigure}

\msg{What we want to show}
At the end of \Cref{sec:rotograd-ii} we discussed that, by casting problem~\ref{eq:stackelberg-rotograd} as a Stackelberg game, we {have convergence} guarantees 
{when} %
the rotation optimizer is the slow-learner. Next, we empirically show this necessary condition. 

\msg{Experimental setup}
\textbf{Experimental setup.} 
Similar to~\citep{sener2018multi}, we use a multitask version of MNIST~\citep{lecun2010mnist} where each image is composed of a left and right digit, and 
use as backbone a reduced version of LeNet~\citep{lecun1998gradient} with light-weight heads. %
Besides the left- and right-digit classification proposed in~\citep{sener2018multi}, we consider three other quantities to predict: i)~sum of digits; ii)~parity of the digit product; and iii)~number of active pixels.
The idea here is to enforce all digit-related tasks to cooperate (positive transfer), while the (orthogonal) image-related task should not disrupt these learning dynamics. %

\msg{Results}
\textbf{Results.} \Figref{fig:learning_rate} shows the effect---averaged over ten independent runs---of changing the rotations' learning rate in terms of test error in the sum task, while the rest of tasks are shown  in \Cref{app:extra-results}. %
We observe that, the bigger the learning rate is, in comparison to that of the network's parameters (\num{1e-3}), the higher and noisier the test error becomes.
MSE keeps decreasing as we lower the learning rate, reaching a sweet-spot at half the network's learning rate (\num{5e-4}).
For smaller values, the rotations' learning is too slow and results start to resemble those of the vanilla case, in which no rotations are applied {(leftmost box in \Figref{fig:learning_rate}).} %

\subsection{\ours building blocks} \label{sec:building-blocks}

In this section, we empirically evaluate to which extent each of the \ours building blocks, that we denote \textit{Scale Only}~(\cref{sec:rotograd-i}) and \textit{Rotate Only}~(\cref{sec:rotograd-ii}), contribute to the performance gain in MTL.

\textbf{Experimental setup.} We test all methods on three different tasks of NYUv2~\citep{DBLP:nyuv2}: 13-class semantic segmentation; depth estimation; and normal surfaces estimation.
To speed up training, all images were resized to $288\times384$ resolution; and data augmentation was applied to alleviate overfitting.
As \MTL architecture, we use SegNet~\citep{DBLP:SegNet} where the decoder is split into three convolutional heads. %
This is the same setup as that of \citet{DBLP:mtan}.

\textbf{Results.} 
The three top rows of \Cref{tab:nyuv2} show the performance of \ours and its both  components in isolation. All the methods with the same number of parameters. 
Compared to Vanilla optimization (4th row), \textit{Rotate Only} improves all metrics 
by homogenizing gradient directions.
\textit{Scale Only} avoids overlooking the normal estimation task and improves on segmentation by homogenizing gradient magnitudes, at the expense of higher depth estimation error.
Remarkably, \ours exploits its scaling and rotation components to obtain the best results in semantic segmentation and depth estimation, while still achieving comparable  performance in the normal estimation task.

{

\sisetup{round-mode = places, round-precision = 2, detect-weight=true,detect-inline-weight=math}
\begin{table*}[t]
    \centering
    \caption{Median (over five runs) on the NYUv2 dataset. \ours obtains great performance in segmentation and depth tasks, and significantly improves the results on normal surfaces. ${\Delta_S}$, ${\Delta_D}$, and ${\Delta_N}$ denote the relative task improvement for each task.} %
    \label{tab:nyuv2}
    \resizebox{\textwidth}{!}{
    \begin{tabular}{ll|S[round-precision=1, table-format=3.1]S[round-precision=1, table-format=2.1]S[round-precision=1, table-format=2.1]|SSSSSSSSS} \toprule
    & & \multicolumn{3}{c|}{\multirow{2}{*}{\shortstack{Relative\\improvements~$\uparrow$}}} & \multicolumn{2}{c}{\multirow{1}{*}{\shortstack{\textbf{S}egmentation~$\uparrow$}}} & \multicolumn{2}{c}{\multirow{1}{*}{\shortstack{\textbf{D}epth~$\downarrow$}}} & \multicolumn{5}{c}{\textbf{N}ormal Surfaces} \\
     & & \multicolumn{3}{c|}{} & \multicolumn{2}{c}{} & \multicolumn{2}{c}{} & \multicolumn{2}{c}{Angle Dist.~$\downarrow$} & \multicolumn{3}{c}{Within $t^{\circ}$~$\uparrow$}  \\ 
    & Method & {${\Delta_S}$} & {${\Delta_D}$} & {${\Delta_N}$} & {\small mIoU} & \multicolumn{1}{c}{\small Pix Acc} & {\small Abs.} & \multicolumn{1}{c}{\small Rel.} & {\small Mean} & {\small Median} & {\small $11.25$} & {\small $22.5$} & \multicolumn{1}{c}{\small $30$}  \\ \midrule
     & Single & 0.00 & 0.00 & \maxf{0.00} & 39.21 & 64.59 & 0.70 & 0.27 & \maxf{25.09} & \maxf{19.18} & \maxf{30.01} & \maxf{57.33} & \maxf{69.30} \\ \midrule \midrule
    \multirow{9}{*}{\rotatebox[origin=c]{90}{\shortstack{With $\mR_k$ $(m=1024)$}}} & Rotate Only & \maxf{3.28} & \maxf{20.49} & \maxf{-6.56} & \maxf{39.63} & \maxf{66.16} & \maxf{0.53} & \maxf{0.21} & \maxf{26.12} & \maxf{20.93} & \maxf{26.85} & \maxf{53.76} & \maxf{66.50} \\ %
    & Scale Only & \maxf{-0.27} & \maxf{20.01} & \maxf{-7.90} & \maxf{38.89} & \maxf{65.94} & 0.54 & 0.22 & \maxf{26.47} & \maxf{21.24} & \maxf{26.24} & \maxf{53.04} & \maxf{65.81} \\
    & \ours & \maxf{1.83} & \maxf{24.04} & \maxf{-6.11} & \maxf{39.32} & \maxf{66.07} & \maxf{0.53} & 0.21 & \maxf{26.01} & \maxf{20.80} & \maxf{27.18} & \maxf{54.02} & \maxf{66.53} \\ \cmidrule{2-14} %
     & Vanilla & -2.66 & 20.58 & -25.70 & 38.05 & 64.39 & 0.54 & 0.22 & 30.02 & 26.16 & 20.02 & 43.47 & 56.87 \\  \cmidrule{2-14}
    & GradDrop & -0.90 & 13.97 & -25.18 & 38.79 & 64.36 & 0.59 & 0.24 & 29.80 & 25.81 & 19.88 & 44.08 & 57.54 \\
    & PCGrad & -2.67 & 20.47 & -26.31 & 37.15 & 63.44 & 0.55 & \maxf{0.22} & 30.06 & 26.18 & 19.58 & 43.51 & 56.87 \\
    & MGDA-UB & -31.23 & -0.65 & \maxf{0.59} & 21.60 & 51.60 & 0.77 & 0.29 & \maxf{24.74} & \maxf{18.90} & \maxf{30.32} & \maxf{57.95} & \maxf{69.88} \\
    & GradNorm & -0.55 & 19.50 & \maxf{-10.45} & 37.22 & 63.61 & 0.54 & 0.22 & \maxf{26.68} & \maxf{21.67} & \maxf{25.95} & \maxf{52.16} & \maxf{64.95} \\
    & IMTL-G & -0.32 & \maxf{17.56} & \maxf{-7.46} & 38.38 & \maxf{64.66} & \maxf{0.54} & 0.22 & \maxf{26.38} & \maxf{21.35} & \maxf{26.56} & \maxf{52.84} & \maxf{65.69} \\ \cmidrule{1-14}
    \multirow{6}{*}{\rotatebox[origin=c]{90}{Without $\mR_k$}} & Vanilla$\ssymbol{2}$ & -0.94 & 16.77 & -25.03 & 37.11 & 63.98 & 0.56 & 0.22 & 29.93 & 25.89 & 20.34 & 43.92 & 57.39 \\
    & GradDrop  & -0.10 & 15.71 & -26.99 & 37.51 & 63.62 & 0.59 & 0.23 & 30.15 & 26.33 & 19.32 & 43.15 & 56.59 \\
    & PCGrad & -0.51 & 19.97 & -24.63 & 38.51 & 63.95 & 0.55 & 0.22 & 29.79 & 25.77 & 20.61 & 44.22 & 57.64 \\
    & MGDA-UB & -32.19 & -8.22 & \maxf{1.50} & 20.75 & 51.44 & 0.73 & 0.28 & \maxf{24.70} & \maxf{18.92} & \maxf{30.57} & \maxf{57.95} & \maxf{69.99} \\
    & GradNorm & 2.18 & \maxf{20.60} & \maxf{-10.23} & 39.29 & \maxf{64.80} & \maxf{0.53} & \maxf{0.22} & \maxf{26.77} & \maxf{21.88} & \maxf{25.39} & \maxf{51.78} & \maxf{64.76} \\
    & IMTL-G & \maxf{1.92} & \maxf{21.35} & \maxf{-6.71} & \maxf{39.94} & \maxf{65.96} & \maxf{0.55} & \maxf{0.21} & \maxf{26.23} & \maxf{21.14} & \maxf{26.77} & \maxf{53.25} & \maxf{66.22} \\ \bottomrule
    \end{tabular}
	}
    \vspace{-10pt}
\end{table*}
}

\subsection{Subspace rotations} \label{sec:exps-subspace}

We now evaluate the effect of subspace rotations as described at the end of \Cref{sec:scaling}, assessing the trade-off between avoiding negative transfer and size of the subspace considered by \ours. 

\textbf{Experimental setup.} 
We test \ours on a 10-task classification problem on CIFAR10~\citep{krizhevsky2009cifar}, using binary cross-entropy and f1-score as loss and metric, respectively, for all tasks.
We use ResNet18~\citep{he2016resnet} {without pre-training} as backbone ($d = 512$), and linear layers with sigmoid activation functions as task-specific heads.

\textbf{Results.} 
\Cref{tab:cifar} (top) shows that rotating the entire space provides the best results, and that these worsen as we decrease the size of $\mR_k$.
Remarkably,
rotating only  {128 features} already outperforms vanilla with no extra per-task parameters (1st row); and rotating 
{256 features} %
already yields comparable results to vanilla optimization with extra capacity (6th row) despite its larger number of (task-specific)  parameters.
These results can be further explained by \Cref{fig:cos_sim_sizes}  in \cref{app:extra-results}, which shows a positive correlation between the size of $\mR_k$ and cosine similarity. %

\subsection{Methods comparison} \label{sec:comparisons}

{
\textbf{Experimental setup.}
In order to provide fair comparisons among methods, all experiments use identical configurations and random initializations. 
For all methods we performed a hyperparameter search and chose the best ones based on validation error. 
Unless otherwise specified, all baselines use \textit{the same architecture (and thus, number of parameters) as \ours}, taking each rotation matrix $\mR_k$ as extra task-specific parameters.
Further experimental details can be found in \cref{app:setups}, as well as extra experiments and complete results in \cref{app:extra-results}.
}

\textbf{NYUv2.} 
\Cref{tab:nyuv2} shows the performance of all baselines with and without the extra capacity. 
\ours significantly improves performance on all tasks  compared with vanilla optimization, and outperforms all other baselines.
{Remarkably, we rotate only \num{1024} dimensions of $\vz$ (out of a total of \num{7} millions) and, as a result, \ours stays on par in training time with the  baselines (around \SI{4}{\hour}, \cref{app:extra-results}).} 
We can also assert the importance of learning the matrices $\mR_k$ properly by  %
comparing in \cref{tab:nyuv2} the different baselines with and without extra capacity. 
\begin{wrapfigure}[14]{r}{.55\textwidth}
    \centering
    \includegraphics[keepaspectratio, width=0.55\textwidth]{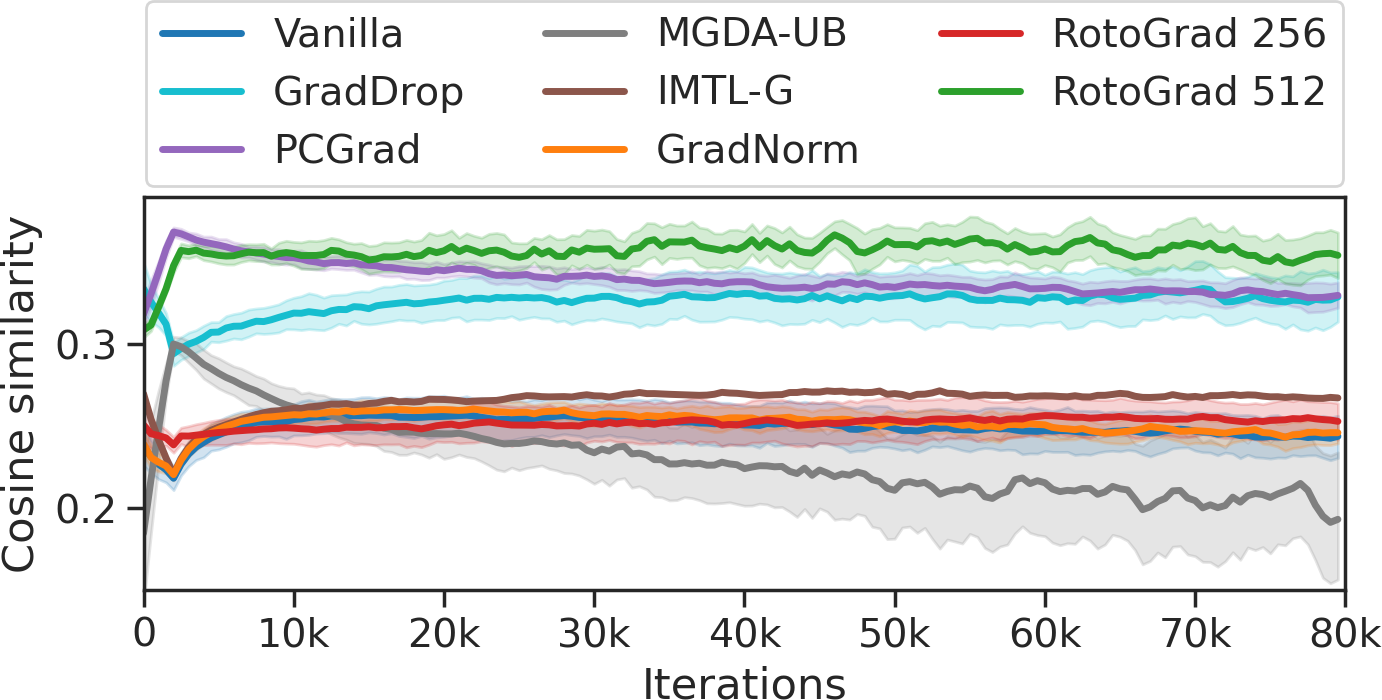}
    \vspace{-1.6em}
    \caption{Similarity between task and update gradients for different methods on CIFAR10, averaged over tasks and five runs.}
    \label{fig:cos_sim_comparisons}
\end{wrapfigure}
This comparison reveals that the extra parameters do not solve the negative transfer 
but instead amplifies biases (methods that overlook a subset of tasks, keep overlooking them) and, in the best case, provides trade-off solutions (also shown in \cref{app:extra-results}).
Note, moreover, that \ours (due to its \textit{Rotate Only} component) is the only method to tackle conflicting gradient directions that manages to not overlook the normal surfaces task.

\textbf{CIFAR10.}
We reuse the setting from \Cref{sec:exps-subspace} 
to compare different \MTL baselines in terms of relative improvements (\cref{tab:cifar}) and cosine similarity (\cref{fig:cos_sim_comparisons}), averaged over five different runs.
\Cref{tab:cifar} (bottom) shows that, similar to the NYUv2 results, both direction-aware solutions (PCGrad and GradDrop) behave similar to vanilla optimization, marginally increasing the average improvement.
Unlike previous experiments, all magnitude-aware methods substantially worsen (at least) one of the statistics.
In contrast, \ours improves the relative task improvement across all statistics using the same number of parameters. %
\Cref{fig:cos_sim_comparisons} shows the cosine similarity between task and update gradients, averaged over all tasks and runs (shaded areas correspond to \SI{90}{\percent} confidence intervals).
It is clear that \ours obtains the highest cosine similarity, that other direction-aware methods also effectively align task gradients and,
combined with the low cosine similarity achieved by MGDA-UB, suggests that \textit{there exists a correlation between cosine similarity and performance}.

{
\setlength\tabcolsep{2.5pt}
\sisetup{table-format=2.2(2), detect-all}
\begin{figure}
    \centering
    \begin{minipage}[]{.55\textwidth}
    \captionof{table}{(Top) Relative task improvement on CIFAR10 for \ours with matrices $\mR_k$ of different sizes; and (bottom) comparison with baseline methods including  rotation matrices as extra task-specific parameters. Table shows median and standard deviation over five runs.}
    \label{tab:cifar}
    \resizebox{\textwidth}{!}{
            \begin{tabular}{ll|SSS} \toprule
    \multicolumn{2}{l|}{Method~~~~~~~~~~~$d$} & {$\E[k]{\Delta_k}$~$\uparrow$} & {$\med_k\Delta_k$~$\uparrow$} & {$\max_k\Delta_k$~$\uparrow$}  \\ \midrule
    \multicolumn{2}{l|}{Vanilla$\ssymbol{2}$~~~~~~~~~~~0} & 2.58 (54) & 1.90 (53) & 11.14 (335) \\ %
    \multicolumn{2}{l|}{\ours\hfill ~64} & 2.90 (49) & 1.79 (57) & 13.16 (240)  \\
    \multicolumn{2}{l|}{\ours\hfill 128} & 2.97 (108) & 2.25 (107) & 12.64 (356) \\ 
    \multicolumn{2}{l|}{\ours\hfill 256} & 3.68 (68) & 2.16 (72) &  14.01 (322) \\
    \multicolumn{2}{l|}{\ours\hfill 512} & \maxf{4.48 (99)} & 3.67 (140) & \maxf{15.57 (399)} \\ \midrule  \midrule
    \multirow{7}{*}{\rotatebox[origin=c]{90}{\shortstack{With $\mR_k$\\$(m = 512)$}}} & Vanilla & 3.12 (79) & 3.10 (129) & 14.23 (286) \\ \cmidrule{2-5}
    & GradDrop & 3.54 (110) & 3.27 (161) & \maxf{13.88 (295)} \\
    & PCGrad & 3.29 (46) & 2.67 (88) & 13.44 (186) \\
    & MGDA-UB & 0.21 (67) & 0.57 (62) & 4.78 (215) \\
    & GradNorm & 3.21 (104) & 3.10 (101) & 10.88 (473) \\
    & IMTL-G & 3.02 (69) & 1.81 (87) & 12.76 (177) \\ %
    \bottomrule
    \end{tabular}
	}%
    \end{minipage}
    \hfill
    \begin{minipage}[]{.425\textwidth}
		\captionof{table}{F1-score statistics in CelebA for two neural network architectures. Median over five different runs. %
		}
		\label{tab:celeba}
		\resizebox{\textwidth}{!}{
			\setlength\tabcolsep{2.5pt}
			\sisetup{table-format=2.2,detect-weight=true}
			\begin{tabular}{ll|SSSS} \toprule
				& & \multicolumn{4}{c}{task f1-scores (\%)~$\uparrow$} \\
				& Method & {$\min_k$} & {$\med_k$} & {$\operatorname{avg}_k$} & {$\operatorname{std}_k$~$\downarrow$} \\ \midrule
				 \multirow{7}{*}{\rotatebox[origin=c]{90}{\shortstack{Conv. Net. with \\ {\footnotesize $\mR_k$ {($m=256$)}}}}} & Vanilla & 4.59 & 50.28 & 56.03 & 25.65 \\ \cmidrule{2-6}
				& GradDrop & 3.18 & 50.07 & 54.43 & 27.21 \\
				& PCGrad & 1.44 & 53.05 & 54.72 & 27.61 \\
				& GradNorm & 2.08 & 52.53 & 56.71 & 24.57 \\
				& IMTL-G & 0.00 & 37.00 & 42.24 & 33.46 \\
				& \ours & 4.59 & 55.02 & 57.20 & 24.75 \\ \midrule \midrule
				\multirow{7}{*}{\rotatebox[origin=c]{90}{\shortstack{ResNet18 with\\ {\footnotesize $\mR_k$ {($m = 1536$)}}}}} & Vanilla & 19.71 & 63.56 & 63.23 & 21.16 \\ \cmidrule{2-6}
				& GradDrop & 12.33 & 62.40 & 62.74 & 21.74 \\
				& PCGrad & 14.71 & 63.65 & 62.61 & 22.22 \\
				& GradNorm & 9.05 & 60.20 & 60.78 & 22.31 \\
				& IMTL-G & 17.11 & 61.68 & 60.72 & 22.80 \\
				& \ours & 9.96 & 63.84 & 62.81 & 21.80 \\
				\bottomrule
			\end{tabular}
		}
    \end{minipage}
\end{figure}
}

\textbf{CelebA.} 
To finalize, we test all methods in a 40-class multi-classification problem on CelebA~\citep{liu2015celeba} and two different settings: one using a convolutional network as backbone ($d = 512$); and another using ResNet18~\citep{he2016resnet} as backbone ($d=2048$).
As above, we use binary cross-entropy and f1-score as loss and metric for all tasks, thus accounting for highly imbalanced attributes.
Results in \Tabref{tab:celeba} show that \ours performs great in all f1-score statistics and  both architectures, specially in the convolutional neural network, outperforming competing methods. 

Moreover, \ours achieves these results rotating \SI{50}{\percent} of the shared feature $\vz$ for the convolutional network, and \SI{75}{\percent} for the residual network, which further demonstrates that \ours can scale-up to real-world settings.
We believe it is important to remark that, due to the high number of tasks, this setup is specially demanding. 
Results in \cref{app:extra-results} show the performance of all baselines without the rotation matrices, demonstrating the negative effect that the extra capacity can have if not learned properly, as well as that \ours stays on par with non-extended baselines in training time. %
\vspace{-2pt}
\section{Conclusions}
\vspace{-3pt}
In this work, we have introduced \ours, an algorithm that tackles negative transfer in \MTL by homogenizing task gradients in terms of both magnitudes and directions.
\ours enforces a similar convergence rate for all tasks, while at the same time smoothly rotates the shared representation differently for each task in order to avoid conflicting gradients. As a result, \ours leads to stable and accurate \MTL.  
Our empirical results have shown the effectiveness of \ours in many scenarios, staying on top of all competing methods in performance, while being on par in terms of computational complexity with those that better scale to complex networks. %

We believe our work opens up interesting venues for future work.
For example, it would be interesting to study alternative approaches to further scale up \ours using, \eg, diagonal-block or sparse rotation matrices; to rotate the feature space in application domains with structured features (e.g., channel-wise rotations in images); and to combine different methods, \eg, by scaling gradients using the direction-awareness of IMTL-G and the ``favor slow-learners'' policy of \ours.
    \bibliography{references}
    \bibliographystyle{iclr2022_conference}

    \newpage
    \appendix
    \begin{appendices}

        \addtocontents{toc}{\protect\setcounter{tocdepth}{1}}
        \makeatletter
        \addtocontents{toc}{%
          \begingroup
          \let\protect\l@chapter\protect\l@section
          \let\protect\l@section\protect\l@subsection
        }
        \makeatother

\section{Proofs} \label{app:proofs}

\begin{proposition}
Suppose $f_k \coloneqq \loss_k \circ \head_k$ is an infinitely differentiable real-valued function, and let us call $\mG_k = \nabla_\mZ f_k(\mZ)$ its derivative with respect to $\mZ$, for every $k = \range{K}$. If $\operatorname{cos\_sim}(\mG_i, \mG_j) > -1/(K-1)$ pairwise; then there exists a small-enough step size $\epsilon > 0$ such that, for all $k$, we have that $\loss_k(\head_k(\mZ - \epsilon \cdot C \sum_k \mU_k; \phib_k); \mY_k) < \loss_k(\head_k(\mZ; \phib_k); \mY_k)$, where $\mU_k \coloneqq \mG_k / \norm{\mG_k}$ and $C \geq 0$.
\end{proposition}
\begin{proof}
Since $f_k$ is infinitely differentiable, we can take the first-order Taylor expansion of $f_k$ around $\mZ$, for any $k$, evaluated at $\mZ - \epsilon \mV$ for a given vector $\mV$:
\begin{equation}
    f_k(\mZ - \epsilon \mV) = f_k(\mZ) - \epsilon \langle \mG_k, \mV \rangle + o(\epsilon).
\end{equation}

In our case, $\mV = C \sum_k \mU_k$ with $C\geq 0$, then:
\begin{align}
    f_k(\mZ - \epsilon \mV) - f_k(\mZ) &= -\epsilon \cdot C \norm{\mG_k} \sum_i \langle \mU_k, \mU_i \rangle + o(\epsilon) \\ 
    & = -\epsilon \cdot C \norm{\mG_k} \left( 1 + \sum_{i\neq j} \langle \mU_k, \mU_i \rangle \right) + o(\epsilon) .
\end{align}

Since $\norm{\mU_k} = 1$ for all $k=\range{K}$, it holds that $-1 \leq \operatorname{cos\_sim}(\mU_k, \mU_i) = \langle \mU_k, \mU_i \rangle \leq 1$.

If $\operatorname{cos\_sim}(\mG_k, \mG_i) > -1/(K-1)$ for all $i\neq k$, then we have that $1 + \sum_{i\neq j} \langle \mU_k, \mU_i \rangle > 0$ and $f_k(\mZ - \epsilon \mV) < f_k(\mV)$ for a small enough $\epsilon > 0$.

\end{proof}

\section{Stackelberg games and \ours} \label{app:stackelberg}

In game theory, a Stackelberg game~\citep{fiez2019convergence} is an \textit{asymmetric game} where two players play alternately.
One of the players---whose objective is to blindly minimize their loss function---is known as the follower, $\follower$.
The other player is known as the leader, $\leader$. In contrast to the follower, the leader attempts to minimize their own loss function, but it does so with the advantage of knowing which will be the best response to their move by the follower. The problem can be written as
\begin{equation}
\begin{aligned}
\mathcal{L}\text{{\small eader}:} & \min_{x_l\in X_l} \{ \leader(x_l, x_f)\,|\,x_f \in \argmin_{y\in X_f} \follower(x_l, y) \}, \\
\mathcal{F}\text{{\small ollower}:} & \min_{x_f\in X_f} \follower(x_l, x_f),  \label{eq:stackelberg}
\end{aligned}
\end{equation}%
where $x_l \in X_l$ and $x_f \in X_f$ are the actions taken by the leader and follower, respectively.

While traditionally one assumes that players make perfect alternate moves in each step of problem~\ref{eq:stackelberg}, \textit{gradient-play Stackelberg games} assume instead that players perform simultaneous gradient updates,
\begin{equation}
\begin{aligned}
x_l^{t+1} &= x_l^t - \alpha_l^t\,\nabla_{x_l}\leader(x_l, x_f), \\
x_f^{t+1} &= x_f^t - \alpha_f^t\,\nabla_{x_f}\leader(x_l, x_f), \\ \label{eq:stackelberg-update-rule}
\end{aligned}
\end{equation}%
where $\alpha_l$ and $\alpha_f$ are the learning rates of the leader and follower, respectively.

An important concept in game theory is that of an equilibrium point, \ie, a point in which both players are satisfied with their situation, meaning that there is no available move immediately improving any of the players' scores, so that none of the players is willing to perform additional actions/updates.
Specifically, we focus on the following definition introduced by \citet{fiez2019convergence}:

\begin{definition}[differential Stackelberg equilibrium]
	A pair of points $x_l^*\in X_l$, $x_f^*\in X_f$, where $x_f^* = r(x_l^*)$ is implicitly defined by $\nabla_{x_f} \follower(x_l^*, x_f^*) = 0$, is a differential Stackelberg equilibrium point if $\nabla_{x_l} \leader(x_l^*, r(x_l^*)) = 0$, and $\nabla^2_{x_l} \leader(x_l^*, r(x_l^*))$ is positive definite.
\end{definition}

Note that, when the players manage to reach such an equilibrium point, both of them are in a local optimum. 
Here, we make use of the following result, introduced by \citet{fiez2019convergence}, to provide theoretical convergence guarantees to an equilibrium point:

\begin{proposition} \label{prop:equilibrium}
	In the given setting, if the leader's learning rate goes to zero at a faster rate than the follower's, that is, $\alpha_l(t) = o(\alpha_f(t))$, where $\alpha_i(t)$ denotes the learning rate of player $i$ at step $t$,
	then they will asymptotically converge to a differential Stackelberg equilibrium point almost surely.
\end{proposition}

In other words, as long as the follower learns faster than the leader, they will end up in a situation where both are satisfied.
Even more, \citet{fiez2019convergence} extended this result to the finite-time case, showing that the game will end close to an equilibrium point with high probability. 

As we can observe, the Stackelberg formulation in \Cref{eq:stackelberg} is really similar to \ours's formulation in \Cref{eq:stackelberg-rotograd}. Even more, the update rule in \Cref{eq:stackelberg-update-rule} is quite similar to that one shown in \Cref{alg:rotograd}.
Therefore, it is sensible to cast \ours as a Stackelberg game.
One important but subtle bit about this link regards the extra information used by the leader. 
In our case, this extra knowledge explicitly appears in \Eqref{eq:roto-loss} in the form of the follower's gradient, $\vg_{i,k}$, which is the direction the follower will follow and, as it is performing first-order optimization by assumption, this gradient directly encodes the follower's response.

Thanks to the Stackelberg formulation in \Eqref{eq:stackelberg-rotograd} we can make use of Prop.~\ref{prop:equilibrium} and, thus, draw theoretical guarantees on the training stability and convergence.
In other words, we can say that performing training steps as those described in \Cref{alg:rotograd} will stably converge as long as the leader is asymptotically the slow learner, \ie $\alpha_l^t = o(\alpha_f^t)$, where $o$ denotes the little-o notation.

In practice, however, the optimization procedure proposed by \citet{fiez2019convergence} requires computing the gradient of a gradient, thus incurring a significant overhead.
Instead, we use Gradient Ascent-Descent (GDA), which only computes partial derivatives and enjoys similar guarantees \citep{jordan-gda}, as we empirically showed in the manuscript.

\section{Experiments} \label{app:experiments}

\subsection{Experimental setups} \label{app:setups}

Here, we discuss common settings across all experiments. Refer to specific sections further below for details concerning single experiments.

\textbf{Computational resources.} All experiments were performed on a shared cluster system with two NVIDIA DGX-A100. Therefore, all experiments were run with (up to) 4 cores of AMD EPYC 7742 CPUs and, for those trained on GPU (CIFAR10, CelebA, and NYUv2), a single NVIDIA A100 GPU. All experiments were restricted to \SI{12}{\giga\byte} of RAM.

\textbf{Loss normalization.} Similar as in the gradient case studied in this work, magnitude differences between losses can make the model overlook some tasks. 
To overcome this issue, here we perform loss normalization, that is, we normalize all losses by their value at the first training iteration (so that they are all 1 in the first iteration). To account for some losses that may quickly decrease at the start, after the 20th iteration, we instead normalize losses dividing by their value at that iteration.

\textbf{Checkpoints.} For the single training of a model, we select the parameters of the model by taking those that obtained the best validation error after each training epoch. That is, after each epoch we evaluate the linearly-scalarized validation loss, $\sum_k \loss_k$, and use the parameters that obtained the best error during training. This can be seen as an extension of early-stopping where, instead of stopping, we keep training until reaching the maximum number of epochs hoping to keep improving.

\textbf{Baselines.} We have implemented all baselines according to the original paper descriptions, except for PCGrad, which we apply to the gradients with respect to the feature $\rvz$ (instead of the shared parameters $\thetab$, as in the original paper). 
Note that this is in accordance with recent works, for example \citet{graddrop} and \citet{liu2021imtl}, which also use this implementation of PCGrad in the feature space.
This way, all competing methods modify gradients with respect to the same variables and, as backpropagation performs the sum of gradients at the last shared representation $\vz$, PCGrad can be applied to reduce conflict at that level.

\textbf{Hyperparameter tuning.} Model-specific hyperparameters were mostly selected by a combination of: i)~using values described in prior works; and ii)~empirical validation on the vanilla case (without any gradient-modifiers) to verify that the combinations of hyperparameters work.
To select method-specific hyperparameters we performed a grid search, choosing those combinations of values that performed the best with respect to validation error.

Specifically, we took $\alpha\in\{0,0.5,1,2\}$ and $\mR_k \in \sR^{m\times m}$ with $m\in\{0.25d, 0.5d, 0.75d, d\}$ (restricting ourselves to $m \leq 1500$) for \ours. Regarding the learning rate of \ours (GradNorm), we performed a grid search considering $\eta_\Roto \in \{0.05\eta, 0.1\eta, 0.5\eta, \eta, 2\eta\}$, where $\eta_\Roto$ and $\eta$ are the learning rates of \ours (GradNorm) and the network, respectively.

\textbf{Statistical test.} For the tabular data, we highlight those results that are significantly better than those from the multitask baseline (that is, better than the vanilla MTL optimization without the $R_k$ matrices). To find these values, we run a paired one-sided Student's t-test across each method and the baseline. For those metrics for which higher is better, our null hypothesis is that the method's performance is equal or lower than the baseline, and for those for which lower is better, the null hypothesis is that the method's performance is equal or greater than the baseline. We use a significance level of $\alpha = 0.05$.

\textbf{Notation.} Along this section, we use the following to describe different architectures: \conv{F}{C} denotes a convolutional layer with filter size $F$ and $C$ number of channels; \maxl denotes a max-pool layer of filter size and stride 2, and \linear{H} a dense layer with output size $H$.

\subsubsection{Illustrative examples}

\textbf{Losses and metrics.} Both illustrative experiments use MSE as both loss and metric. Regarding the specific form of $\varphi$ in \Cref{eq:illustrative}, the avocado-shaped experiments uses
\begin{equation}
    \varphi((x, y), s) = (x - s)^2 + 25y^2,
\end{equation}
while the non-convex second experiment uses
\begin{equation}
    \varphi((x,y), s) = -\frac{\sin(3x + 4.5s)}{x + 1.5s} - \frac{\sin(3y + 4.5s)}{y + 1.5s} + \abs{x + 1.5s} + \abs{y + 1.5s}
\end{equation}

\textbf{Model.} As described in the main manuscript, we use a single input $\vx\in\sR^2$ picked at random from a standard normal distribution, and drop all task-specific network parameters (that is, $\head_k(\vr_k) = \vr_k$).
As backbone, we take a simple network of the form $\vz = \mW_2 \max(\mW_1\vx + \vb_1, 0) + \vb_2$ with $\vb_1\in\sR^{10}, \vb_2\in\sR^2$, and $\mW_1, \mW_2^\top \in \sR^{10\times2}$.

\textbf{Hyperparameters, convex-case.} We train the model for one hundred epochs. As network optimizer we use stochastic gradient descent (SGD) with a learning rate of \num{0.01}. For the rotations we use RAdam~\citep{liu2019radam} with a learning rate of \num{0.5} (for visualization purposes we need a high learning rate, in such a simple scenario it still converges) and exponential decay with decaying factor \num{0.99999}.

\textbf{Hyperparameter, non-convex case.} For the second experiment, we train the model for 400 epochs and, once again, use SGD as the network optimizer with a learning rate of \num{0.015}. For the rotations, we use RAdam~\citep{liu2019radam} with a learning rate of \num{0.1} and an exponential decay of \num{0.99999}.

\subsubsection{MNIST/SVHN} \label{subsec:exp-setup-mnist}

\textbf{Datasets.} We use two modified versions of MNIST~\citep{lecun2010mnist} and SVHN~\citep{netzer2011reading} in which each image has two digits, one on each side of the image. In the case of MNIST, both of them are merged such that they form an overlapped image of $28\times28$, as shown in \Cref{fig:mnist-sample}. Since SVHN contains backgrounds, we simply paste two images together without overlapping, obtaining images of size $32\times64$, as shown in \Cref{fig:svhn-sample}. Moreover, we transform all SVHN samples to grayscale.

\begin{figure}[htbp]
    \centering
    \subcaptionbox{Multi-MNIST.\label{fig:mnist-sample}}{\includegraphics[keepaspectratio, width=.33\textwidth]{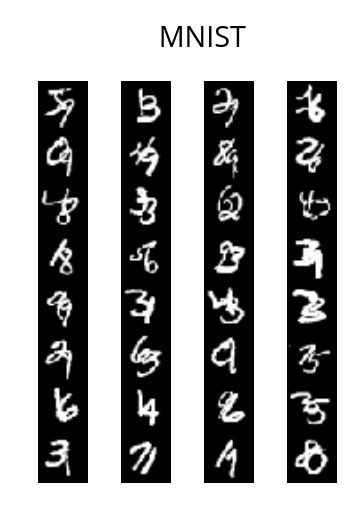}} %
    \subcaptionbox{Multi-SVHN.\label{fig:svhn-sample}}{\includegraphics[keepaspectratio, width=.49\textwidth]{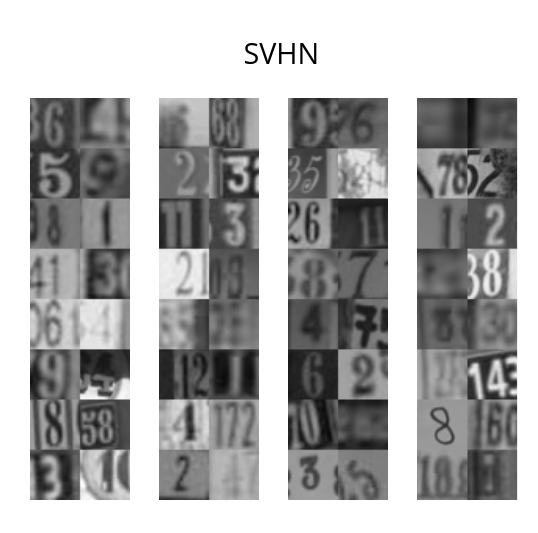}} %
    \caption{Samples extracted from the modified MNIST and SVHN datasets.}
    \label{fig:datasets}
\end{figure}

\textbf{Tasks, losses, and metrics.} In order to further clarify the setup used, here we describe in detail each task. Specifically, we have:
\begin{itemize}
    \item \textbf{Left digit} classification. Loss: negative log-likelihood (NLL). Metric: accuracy (ACC).
    \item \textbf{Right digit} classification. Loss: NLL. Metric: ACC.
    \item \textbf{Parity} of the product of digits, \ie, whether the product of both digits gives an odd number (binary-classification). Loss: binary cross entropy (BCE). Metric: f1-score (F1).
    \item \textbf{Sum} of both digits (regression). Loss: MSE. Metric: MSE.
    \item \textbf{Active pixels} in the image, \ie, predict the number of pixels with values higher than \num{0.5}, where we use pixels lying in the unit interval (regression). Loss: MSE. Metric: MSE.
\end{itemize}

\textbf{Model.} 
Our backbone is an adaption from the original LeNet~\citep{lecun1998gradient} model. Specifically, we use:
\begin{itemize}
    \item \textbf{MNIST.} \conv{5}{10}\maxl\relu\conv{5}{20}\maxl\linear{50}\relu\batchnorm ,
    \item \textbf{SVHN.} \conv{5}{10}\maxl\relu\conv{5}{20}\maxl\conv{5}{20}\linear{50}\\*\relu\batchnorm ,
\end{itemize}
where \batchnorm refers to Batch Normalization~\citep{ioffe2015batch}. 
Depending on the type of task, we use a different head. Specifically, we use:
\begin{itemize}
    \item \textbf{Regression.} \linear{50}\relu\linear{1},
    \item \textbf{Classification.} \linear{50}\relu\linear{10}\logsoftmax,
    \item \textbf{Binary-classification.} \linear{1}\sigmoidl.
\end{itemize}

\textbf{Model hyperparameters.} For both datasets, we train the model for \num{300} epochs using a batch size of \num{1024}. For the network parameters, we use RAdam~\citep{liu2019radam} with a learning rate of \num{1e-3}.

\textbf{Methods hyperparameters.} In \Cref{tab:mnist-svhn,tab:mnist-svhn-full} we show the results of GradNorm with: i)~MNIST with $R_k$, $\alpha = 0$; ii)~MNIST without $R_k$, $\alpha = 0.5$; iii)~SVHN with $R_k$, $\alpha = 1$; and iv)~SVHN without $R_k$, $\alpha = 2$. We train \ours with full-size rotation matrices ($m = d$). Both methods use RAdam with learning rate \num{5e-4} and exponential decay of \num{0.9999}.

\subsubsection{CIFAR10}

\textbf{Dataset.} We use CIFAR10~\citep{krizhevsky2009cifar} as dataset, with \num{40000} instances as training data and the rest as testing data. Additionally, every time we get a sample from the dataset we: i)~crop the image by a randomly selected square of size $3\times32\times32$; ii)~randomly flip the image horizontally; and iii)~standardize the image channel-wise using the mean and standard deviation estimators obtained on the training data.

\textbf{Model.} We take as backbone ResNet18~\citep{he2016resnet} without pre-training, where we remove the last linear and pool layers. In addition, we add a Batch Normalization layer. For each task-specific head, we simply use a linear layer followed by a sigmoid function, that is, \linear{1}\sigmoidl.

\textbf{Losses and metrics.} We treat each class (out of ten) as a binary-classification task where we use BCE and F1 as loss and metric, respectively.

\textbf{Model hyperparameters.} We use a batch size of \num{128} and train the model for \num{500} epochs. For the network parameters, we use as optimizer SGD with learning rate of \num{0.01}, Nesterov momentum of \num{0.9}, and a weight decay of \num{5e-4}. Additionally, we use for the network parameters a cosine learning-rate scheduler with a period of \num{200} iterations. 

\textbf{Methods hyperparameters.} Results shown in \Cref{tab:cifar,tab:cifar-full} use $\alpha=0$ and $\alpha = 0.5$ for GradNorm with and without $R_k$, respectively, and we use RAdam~\citep{liu2019radam} as optimizer with learning rate \num{0.001} and an exponential decay factor of \num{0.99995} for both GradNorm and \ours.

\subsubsection{NYUv2}

\textbf{Setup.}
In contrast with the rest of experiments, for the NYUv2 experiments shown in \Cref{sec:experiments}, instead of writing our own implementation, we slightly modified the open-source code provided by \citet{DBLP:mtan} at \url{https://github.com/lorenmt/mtan} (commit 268c5c1).
We therefore use the exact same setting as \citet{DBLP:mtan}---and refer to their paper and code for further details, with the addition of using data augmentation for the experiments which, although not described in the paper, is included in the repository as a command-line argument.
We will provide along this work a diff file to include all gradient-modifier methods into the aforementioned code.

\textbf{Methods hyperparameters.} For the results shown in \Cref{tab:nyuv2} we use GradNorm with $\alpha=0$ and \ours with rotations $\mR_k$ of size \num{1024}. We use a similar optimization strategy as the rest of parameters, using Adam~\citep{kingma2014adam} with learning rate \num{5e-5} (half the one of the network parameters) and where we halve this learning rate every 100 iterations.

\subsubsection{CelebA}

\textbf{Dataset.} We use CelebA~\citep{liu2015celeba} as dataset with usual splits. We resize each sample image so that they have size $3\times64\times64$.

\textbf{Losses and metrics.} We treat each class (out of forty) as a binary-classification task where we use BCE and F1 as loss and metric, respectively.

\textbf{ResNet model.} As with CIFAR10, we use as backbone ResNet18~\citep{he2016resnet} without pre-training, where we remove the last linear and pool layers. In addition, we add a Batch Normalization layer. For each task-specific head, we use a linear layer followed by a sigmoid function, that is, \linear{1}\sigmoidl.

\textbf{ResNet hyperparameters.} We use a batch size of \num{256} and train the model for \num{100} epochs. For the network parameters, we use RAdam~\citep{liu2019radam} as optimizer with learning rate \num{0.001} and exponential decay of \num{0.99995} applied every \num{2400} iterations.

\textbf{Convolutional model.} For the second architecture, we use a convolutional network as backbone, \conv{3}{64}\batchnorm\maxl\conv{3}{128}\batchnorm\conv{3}{128}\batchnorm\maxl\conv{3}{256}\batchnorm\\*\conv{3}{256}\batchnorm\maxl\conv{3}{512}\batchnorm\linear{512}\batchnorm. For the task-specific heads, we take a simple network of the form \linear{128}\batchnorm\linear{1}\sigmoidl.

\textbf{Convolutional hyperparameters.} We use a batch size of \num{8} and train the model for \num{20} epochs. For the network parameters, we use RAdam~\citep{liu2019radam} as optimizer with learning rate \num{0.001} and exponential decay of \num{0.96} applied every \num{2400} iterations.

\textbf{Methods hyperparameters.} Results shown in \Cref{tab:celeba,tab:celeba-full} use GradNorm with: i)~convolutional network with $R_k$, $\alpha=0$; ii)~convolutional network without $R_k$, $\alpha=1$; iii)~residual network with $R_k$, $\alpha=0.5$; and iv)~residual network without $R_k$, $\alpha=1$. For \ours, we rotate \num{256} and \num{1536} elements of $\vz$ for the convolutional and residual networks. As optimizer, we use RAdam~\citep{liu2019radam} with learning rate \num{5e-6} and an exponential decay factor of \num{0.99995} for both GradNorm and \ours. 

Note that for these experiments we omit MGDA-UB~\citep{sener2018multi} as it is computationally prohibitive in comparisons with other methods. In single-seed experiments, we however observed that it does not perform too well (specially in the convolutional network).

\subsection{Additional results} \label{app:extra-results}

\subsubsection{Multi-MNIST and multi-SVHN}

\begin{table}[h] %
    \centering
    \caption{Test performance (median and standard deviation) on two set of unrelated tasks on MNIST and SVHN, across ten different runs.}
    \label{tab:mnist-svhn}
    \setlength\tabcolsep{2.5pt}
    \sisetup{table-format=2.2(2)}
    \begin{tabular}{l|S[table-format=4.2(3)]SSS} \toprule
    & \multicolumn{2}{c}{MNIST} & \multicolumn{2}{c}{SVHN} \\ 
    & {Digits} & {Act Pix} & {Digits} & {Act Pix}  \\ 
    Method & {$\E[k]{\Delta_k}$~$\uparrow$} & {MSE~$\downarrow$} & {$\E[k]{\Delta_k}$~$\uparrow$} & {MSE~$\downarrow$} \\ 
    \midrule
    Single  & \maxf{0.00 (0)} & \maxf{0.01 (1)} & 0.00 (0) & \maxf{0.17 (6)} \\
    Vanilla & -1.43 (324) & 0.14 (5) & 4.78 (88) & 3.04 (253) \\
    \midrule
    GradDrop & -1.30 (182) & 0.16 (4) & \maxf{5.34 (92)} & 2.99 (259) \\
    PCGrad & -1.22 (281) & 0.13 (1) & 5.01 (65) & 2.70 (225) \\
    MGDA-UB & -29.14 (923) & \maxf{0.06 (0)} & -4.36 (672) & \maxf{1.00 (57)} \\
    GradNorm & \maxf{0.86 (193)} & 0.09 (4) & \maxf{5.24 (89)} & 4.12 (946) \\
    IMTL-G & \maxf{2.12 (146)} & \maxf{0.07 (2)} & \maxf{5.94 (99)} & 1.70 (105) \\
    \ours & \maxf{1.55 (222)} & \maxf{0.08 (3)} & \maxf{6.08 (48)} & 1.61 (272) \\ \bottomrule
    \end{tabular}
\end{table}
We reuse the experimental setting from \Cref{sec:stability}---now using the original LeNet~\citep{lecun1998gradient} and a multitask-version of SVHN~\citep{netzer2011reading}---in order to evaluate how disruptive the orthogonal image-related task is for different methods. 
We can observe (\cref{tab:mnist-svhn}) that %
the effect of the  image-related task is more disruptive in MNIST, in which MGDA-UB utterly fails. 
Direction-aware methods (GradDrop and PCGrad) do not improve the vanilla results, whereas IMTL-G, GradNorm, and \ours obtain the best results. %
We also provide the complete results for all metrics in \Cref{tab:mnist-svhn-full}.
In the case of MNIST, we can observe that both regression tasks tend to be quite disruptive. GradNorm, IMTL-G, and \ours manage to improve over all tasks while maintaining good performance on the rest of tasks. MGDA-UB, however, focuses on the image-related task too much and overlooks other tasks.
In SVHN we observe a similar behavior. This time, all methods are able to leverage positive transfer and improve their results on the parity and sum tasks, obtaining similar task improvement results. 
Yet, the image-related task is more disruptive than before, showing bigger differences between methods. As before, MGDA-UB completely focuses on this task, but now is able to not overlook any task while doing so. Regarding the rest of the methods, all of them improved their results with respect to the vanilla case, with \ours and GradNorm obtaining the second-best results.

{
\sisetup{table-format=2.2(2), detect-all}
\begin{table}[hbtp]
    \centering
    \caption{Complete results (median and standard deviation) of different competing methods on MNIST/SVHN on all tasks, see~\Cref{subsec:exp-setup-mnist} and~\Cref{app:extra-results}.}
    \label{tab:mnist-svhn-full}
    \resizebox{\textwidth}{!}{
    \begin{tabular}{ccl|SSSSS|S} \toprule
         & & & {Left digit} & {Right digit} & {Product parity} & {Sum digits} & & {Act. Pix.} \\
         & & Method & {Acc.~$\uparrow$} & {Acc.~$\uparrow$} & {f1~$\uparrow$} & {MSE~$\downarrow$} & {$\E[k]{\Delta_k}$~$\uparrow$} & {MSE~$\downarrow$} \\
          \midrule
        \multirow{15}{*}{\rotatebox[origin=c]{90}{MNIST}} & \multirow{7}{*}{\rotatebox[origin=c]{90}{Without $\mR_k$}} & Single & \maxf{95.70 (20)} & \maxf{94.05 (16)} & 92.09 (76) & \maxf{1.90 (10)} & \maxf{0.00 (0)} & \maxf{0.01 (1)} \\ \cmidrule{3-9}
        & & Vanilla$\ssymbol{2}$ & 94.94 (20) & 93.26 (27) & 93.07 (48) & 2.10 (17) & -3.26 (312) & 0.11 (1) \\
        & & GradDrop & \maxf{95.33 (39)} & \maxf{93.55 (29)} & 93.32 (54) & 2.14 (7) & -2.52 (163) & 0.13 (2) \\
        & & PCGrad & 95.07 (39) & 93.28 (18) & \maxf{93.34 (51)} & 2.14 (19) & -3.36 (386) & 0.12 (2) \\
        & & MGDA-UB & 94.46 (104) & 92.23 (154) & 83.89 (184) & 2.50 (60) & -10.80 (1045) & \maxf{0.06 (2)} \\
        & & GradNorm & 95.19 (37) & \maxf{93.70 (31)} & 93.31 (39) & 2.06 (2871) & -1.81 (3799) & 0.09 (746) \\
        & & IMTL-G & \maxf{95.28 (38)} & \maxf{93.84 (21)} & 93.24 (49) & 1.91 (661) & -0.01 (8248) & 0.07 (205) \\ \cmidrule{3-9}
        & \multirow{7}{*}{\rotatebox[origin=c]{90}{With $\mR_k$}} & Vanilla & 95.13 (20) & \maxf{93.41 (17)} & \maxf{93.54 (50)} & 1.99 (17) & -1.43 (324) & 0.14 (5) \\
        & & GradDrop & \maxf{95.14 (16)} & \maxf{93.47 (12)} & \maxf{93.59 (32)} & 2.00 (6) & -1.30 (182) & 0.16 (4) \\
        & & PCGrad & 95.04 (26) & \maxf{93.36 (30)} & 93.49 (30) & 1.98 (13) & -1.22 (281) & 0.13 (1) \\
        & & MGDA-UB & 89.99 (221) & 86.76 (118) & 79.24 (283) & 3.65 (42) & -29.14 (923) & \maxf{0.06 (0)} \\
        & & GradNorm & \maxf{95.28 (18)} & \maxf{93.56 (25)} & 93.56 (57) & \maxf{1.86 (7)} & \maxf{0.86 (193)} & 0.09 (4) \\
        & & IMTL-G & \maxf{95.47 (27)} & \maxf{93.79 (31)} & 93.56 (57) & \maxf{1.73 (9)} & \maxf{2.12 (146)} & \maxf{0.07 (2)} \\
        & & \ours & \maxf{95.45 (19)} & \maxf{93.83 (19)} & 93.22 (35) & \maxf{1.85 (13)} & \maxf{1.55 (222)} & \maxf{0.08 (3)} \\
        \midrule \midrule
        \multirow{15}{*}{\rotatebox[origin=c]{90}{SVHN}} & \multirow{7}{*}{\rotatebox[origin=c]{90}{Without $\mR_k$}} & Single & \maxf{85.05 (45)} & \maxf{84.58 (24)} & 77.47 (113) & 5.84 (14) & 0.00 (0) & \maxf{0.17 (6)} \\ \cmidrule{3-9}
         & & Vanilla$\ssymbol{2}$ & 84.18 (30) & 84.18 (38) & 80.11 (85) & 4.81 (6) & 5.14 (83) & 2.75 (317) \\
         & & GradDrop & \maxf{84.38 (29)} & 84.48 (41) & 80.11 (69) & \maxf{4.69 (12)} & \maxf{5.68 (105)} & 1.91 (86) \\
         & & PCGrad& \maxf{84.22 (31)} & 84.23 (21) & 79.92 (79) & \maxf{4.69 (9)} & \maxf{5.50 (75)} & 2.26 (85) \\
         & & MGDA-UB & 84.61 (75) & 84.38 (45) & 77.44 (144) & \maxf{4.47 (18)} & 5.99 (148) & \maxf{0.66 (75)} \\
         & & GradNorm  & 84.23 (33) & 84.13 (30) & 79.40 (87) & 4.92 (7) & 4.60 (101) & 4.30 (218) \\
         & & IMTL-G & 84.60 (45) & 84.39 (37) & 79.63 (110) & \maxf{4.57 (13)} & \maxf{5.81 (85)} & 2.47 (165) \\ \cmidrule{3-9}
         & \multirow{7}{*}{\rotatebox[origin=c]{90}{With $\mR_k$}} & Vanilla & 84.11 (48) & 84.11 (40) & 79.83 (79) & 4.84 (10) & 4.78 (88) & 3.04 (253) \\
         & & GradDrop & 84.23 (35) & 84.33 (40) & 80.10 (83) & \maxf{4.73 (9)} & \maxf{5.34 (92)} & 2.99 (259) \\
         & & PCGrad & 84.21 (21) & 84.26 (38) & 79.64 (48) & 4.84 (6) & 5.01 (65) & 2.70 (225) \\
         & & MGDA-UB & 77.05 (544) & 78.00 (504) & 71.76 (432) & 5.27 (56) & -4.36 (672) & \maxf{1.00 (57)} \\
         & & GradNorm & \maxf{84.37 (34)} & 84.30 (46) & 79.97 (75) & \maxf{4.72 (13)} & \maxf{5.24 (89)} & 4.12 (946) \\
         & & IMTL-G & 84.23 (34) & 84.23 (39) & 79.77 (104) & \maxf{4.51 (12)} & \maxf{5.94 (99)} & 1.70 (105) \\
         & & \ours & \maxf{84.60 (50)} & 84.44 (45) & 79.14 (96) & \maxf{4.45 (10)} & \maxf{6.08 (48)} & 1.61 (272) \\
         \bottomrule
    \end{tabular}
    }
\end{table}
}

\newpage

\subsubsection{Illustrative examples}

\begin{wrapfigure}{r}{.55\textwidth}
\vspace{-35pt}
    \centering
    \includegraphics[width=.249\textwidth, keepaspectratio]{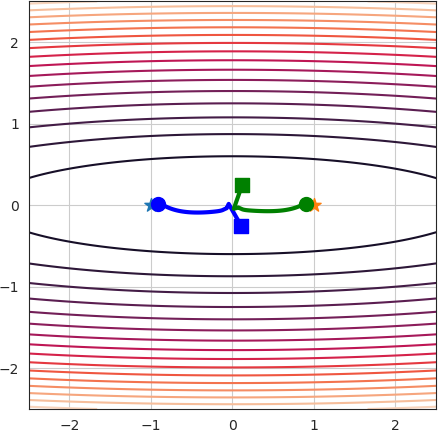}
    \includegraphics[width=.249\textwidth, keepaspectratio]{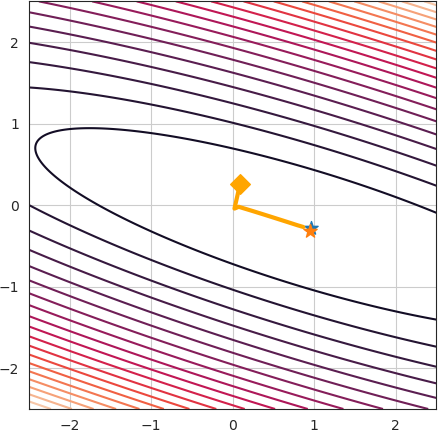}
    \includegraphics[width=.249\textwidth, keepaspectratio]{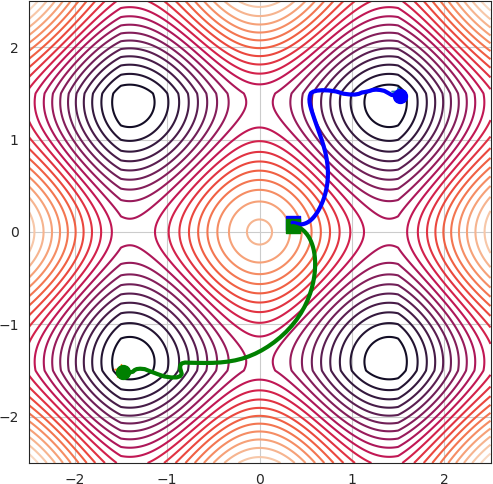}
    \includegraphics[width=.249\textwidth, keepaspectratio]{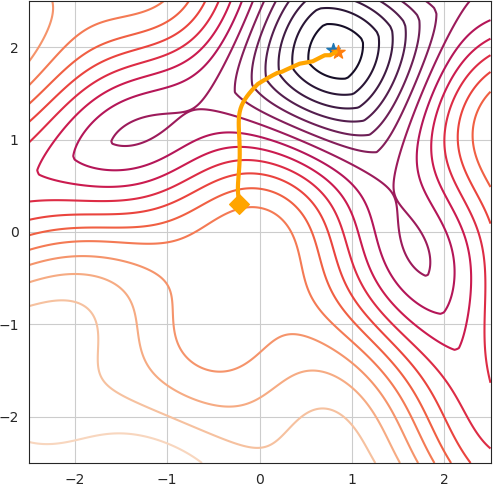}
    \caption{Level plots showing the illustrative examples of \Cref{fig:qualitative-plots} for \ours. Top: Convex case. Bottom: Non-convex case. Left: Active transformation (trajectories of $\vr_k$ and the level plot of $\loss_1 + \loss_2$. Right: Passive transformation (trajectory of $\vz$ and level plot of $(\loss_1 \circ \mR_1) + (\loss_2 \circ \mR_2)$).}
    \label{fig:qualitative-full}
    \vspace{-10pt}
\end{wrapfigure}
We complement the illustrative figures shown in \Cref{fig:qualitative-plots} by providing, for each example, an illustration of the effect of \ours shown as an active and passive transformation.
In an active transformation (\Cref{fig:qualitative-full} left), points in the space are the ones modified. In our case, this means that we rotate feature $\vz$, obtaining $\vr_1$ and $\vr_2$, while the loss functions remain the same. In other words, for each $\vz$ we obtain a task-specific feature $\vr_k$ that optimizes its loss function.
In contrast, a passive transformation (\Cref{fig:qualitative-full} right) keeps the points unaltered while applying the transformation to the space itself. In our case, this translates to rotating the optimization landscape of each loss function (now we have $\loss_k \circ \mR_k$ instead of $\loss_K$), so that our single feature $\vz$ has an easier job at optimizing both tasks. In the case of \ours, we can observe in both right figures that both optima lie in the same point, as we are aligning task gradients.

\begin{wrapfigure}{r}{.55\textwidth}%
    \centering
    \includegraphics[width=.55\textwidth, keepaspectratio]{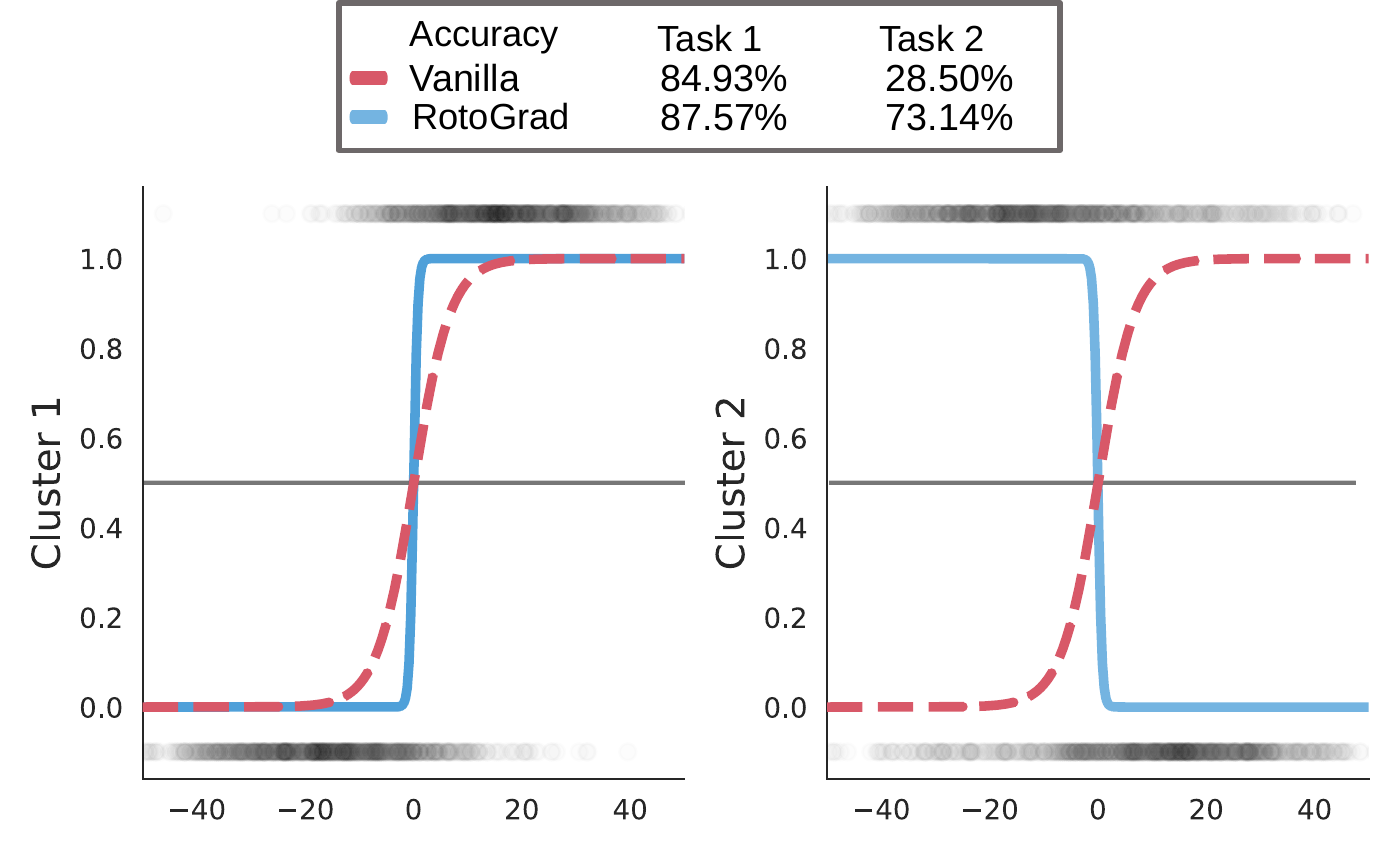}
    \caption{Logistic regression for opposite classification tasks. Test data is plotted scattered as gray dots. \ours learns both opposite rotations $\mR_1 = \mR_2^\top$.}
    \label{fig:toy3}
\end{wrapfigure}

Besides the two regression experiments shown in \Cref{sec:illustrative-examples}, we include here an additional experiment where we test \ours in the worst-case scenario of gradient conflict, \ie, one in which task gradients are opposite to each other. 
To this end, we solve a 2-task binary classification problem where, as dataset, we take 1000 samples from a 2D Gaussian mixture model with two clusters; $\vy_{n,k} = 1$ if $\vx_n$ was sampled from cluster $k$; and $\vy_{n,k} = 0$ otherwise.
We use as model a logistic regression model of the form $\vy_k = \mW_2 \max(\mW_1\vx + \vb_1, 0) + \vb_2$ with $\vb_1\in\sR^{2}, \vb_2\in\sR$, $\mW_1\in \sR^{2\times2}$, and $\mW_2\in\sR^{1\times2}$. Because rotations in 1D are ill-posed (there is a unique proper rotation), here we add task parameters to increase the dimensionality of $\vz$ and make \textit{all parameters shared}, so that there is still no task-specific parameters.
To avoid a complete conflict where $\nabla_\vz \loss_1 + \nabla_\vz \loss_2 = 0$, we randomly flip the labels for the second tasks with \SI{5}{\percent} probability.
\Figref{fig:toy3} shows that, in this extreme scenario, \ours is able to learn both tasks by aligning gradients, \ie, by learning that one rotation is the inverse of the other $\mR_1 = \mR_2^\top$.

\subsubsection{Training stability}

While we showed in \Cref{sec:stability} only the results for the sum-of-digits task as they were nice and clear, here we show in \Figref{fig:lr_all_tasks} the results of those same experiments in \Cref{sec:stability} for all the different tasks.
The same discussion from the main manuscript can be carried out for all metrics. Additionally, we can observe that the vanilla case (learning rate zero) completely overlooks the image-related task (\textit{Active pixels}) while performing the best in the parity task.

Additionally, let us clarify what we mean here with stability, as in the main manuscript we mainly talked about convergence guarantees. 
In these experiments we measure the convergence guarantees of the experiments in terms of `training stability', meaning the variance of the obtained results across different runs.
The intuition here is that, since the model does not converge, we should expect some wriggling learning curves during training and, as we take the model with the best validation error, the individual task metrics should have bigger variance (\ie, less stability) across runs.

\begin{figure}[ht]
    \centering
    \includegraphics[keepaspectratio, width=.95\textwidth]{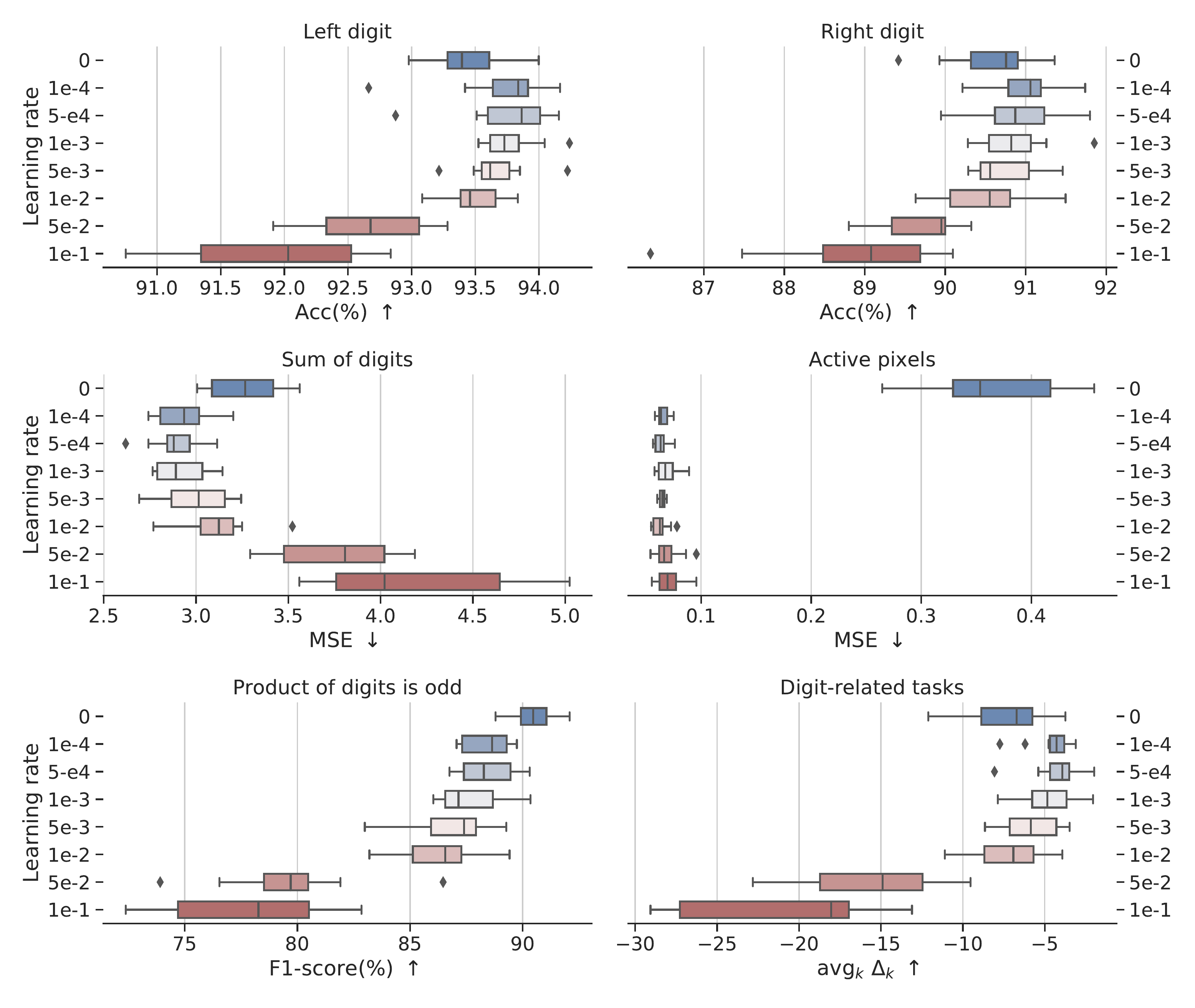}
    \caption{\ours's performance on all tasks for the experiments in \Cref{sec:stability} for all metrics. We can observe training instabilities/stiffness on all tasks as we highly increase/decrease \ours's learning rate, as discussed in the main manuscript.}
    \label{fig:lr_all_tasks}
\end{figure}

\subsubsection{CIFAR10 and CelebA}

For the sake of completeness, we present in \Cref{tab:cifar-full} and \Cref{tab:celeba-full} the same tables as in \Cref{sec:experiments}, but with more statistics of the results.
For CIFAR10, we now included in \Cref{tab:cifar-full} the minimum task improvement across tasks and, while noisier, we can still observe that \ours also improve this statistic. The standard deviation of the task improvement across tasks is, however, not too informative.
We also include in \Cref{fig:cos_sim_sizes} the cosine similarity plots for the different subspace rotations from \Cref{sec:exps-subspace}, showing a clear positive correlation between the cosine similarity and the size of the considered subspace.
While the cosine similarities look low, we want to remark that we are computing the cosine similarity of a huge space, and we only align a subspace of it. If we, instead, showed the cosine similarity with respect to each specific subspace, the cosine similarity should look similar to that of \ours 512.
In the case of CelebA, we added in \Cref{tab:celeba-full} the maximum f1-score across tasks and, similar to the last case, it is not too informative, as all methods achieve almost perfect f1-score in one of the classes.
We also include the training times for some baselines, showing that \ours stays on par with them. %

\begin{figure}[t]
    \centering
    \includegraphics[keepaspectratio, width=.8\textwidth]{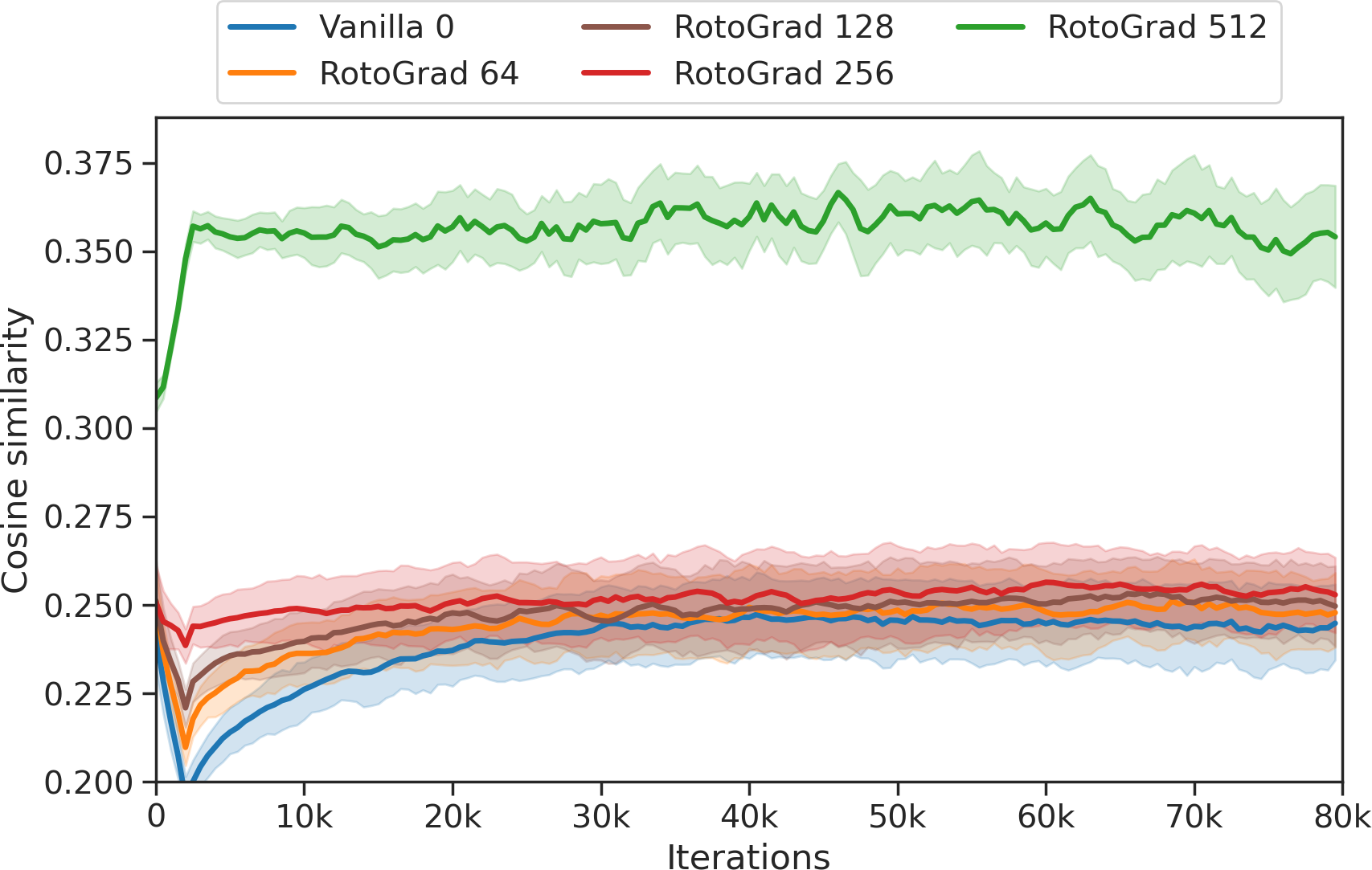}
    \caption{Cosine similarity between the task gradients and the update gradient on CIFAR10. Results are averaged over tasks and five runs, showing \SI{90}{\percent} confidence intervals (shaded areas).}
    \label{fig:cos_sim_sizes}
\end{figure}

{
\sisetup{table-format=2.2(2), detect-all}
\begin{table}[hbtp]
    \centering
    \caption{Complete task-improvement statistics in CIFAR10 for all competing methods and \ours with different dimensionality for $\mR_k$, see \Cref{sec:experiments}.}
    \label{tab:cifar-full}
    \resizebox{\textwidth}{!}{
            \begin{tabular}{ll|SSSSS} \toprule
    \multicolumn{2}{l|}{Method\hfill$d$} & {$\min_k{\Delta_k}$~$\uparrow$} & {$\med_k\Delta_k$~$\uparrow$} & {$\E[k]{\Delta_k}$~$\uparrow$} & {$\std_k\Delta_k$~$\downarrow$} & {$\max_k\Delta_k$~$\uparrow$}  \\ \midrule
    \multicolumn{2}{l|}{Vanilla$\ssymbol{2}$\hfill 0} & -0.81 (37) & 1.90 (53) & 2.58 (54) & 3.38 (94) & 11.14 (335) \\
    \multicolumn{2}{l|}{\ours\hfill 64} & -1.70 (81) & 1.79 (57) & 2.90 (49) & 3.98 (62) & 13.16 (240) \\
    \multicolumn{2}{l|}{\ours\hfill 128} & -1.12 (36) & 2.25 (107) & 2.97 (108) & 3.84 (87) & 12.64 (356) \\
    \multicolumn{2}{l|}{\ours\hfill 256} & 0.17 (101) & 2.16 (72) & 3.68 (68) & 3.83 (74) & 14.01 (322) \\
    \multicolumn{2}{l|}{\ours\hfill 512} & -0.43 (76) & 3.67 (140) & \maxf{4.48 (99)} & 4.23 (82) & \maxf{15.57 (399)} \\ \midrule %
    \multirow{6}{*}{\rotatebox[origin=c]{90}{Without $\mR_k$}} & Vanilla$\ssymbol{2}$ & -0.81 (37) & 1.90 (53) & 2.58 (54) & 3.38 (94) & 11.14 (335) \\ 
    & GradDrop & -0.73 (33) & \maxf{2.80 (20)} & \maxf{3.41 (45)} & 4.08 (34) & \maxf{13.58 (150)} \\
    & PCGrad & -1.52 (98) & 1.95 (87) & 2.86 (81) & 3.74 (69) & 12.01 (319) \\
    & MGDA-UB & -7.27 (136) & -1.21 (74) & -1.75 (43) & 3.24 (55) & 3.67 (98) \\
    & GradNorm & -0.35 (59) & \maxf{2.45 (66)} & 3.23 (35) & 4.02 (33) & 14.25 (135) \\
    & IMTL-G & -0.39 (82) & 1.97 (29) & 2.73 (27) & 3.25 (75) & 10.20 (298) \\ \cmidrule{2-7}
    \multirow{7}{*}{\rotatebox[origin=c]{90}{\shortstack{With $\mR_k$\\$(d = 512)$}}} & Vanilla & -0.85 (58) & 3.10 (129) & 3.12 (79) & 4.05 (56) & \maxf{14.23 (286)} \\
    & GradDrop & -1.49 (78) & 3.27 (161) & 3.54 (110) & 4.11 (56) & \maxf{13.88 (295)} \\
    & PCGrad & -1.44 (58) & 2.67 (88) & 3.29 (46) & 3.90 (37) & 13.44 (186) \\
    & MGDA-UB & -3.59 (148) & 0.57 (62) & 0.21 (67) & \maxf{2.44 (52)} & 4.78 (215) \\
    & GradNorm & -0.79 (128) & \maxf{3.10 (101)} & 3.21 (104) & 3.41 (86) & 10.88 (473) \\
    & IMTL-G & -1.29 (52) & 1.81 (87) & 3.02 (69) & 3.81 (21) & 12.76 (177) \\
    & \ours\hfill & -0.43 (76) & 3.67 (140) & \maxf{4.48 (99)} & 4.23 (82) & \maxf{15.57 (399)} \\ \bottomrule
    \end{tabular}
	}
\end{table}
}

\begin{table}[hbtp] %
    \centering 
    \caption{Complete f1-score statistics and training hours in CelebA for all competing methods and two different architectures/settings (median over five runs), see \Cref{sec:experiments}. For the convolutional network we use $m=256$, and $m=1536$ for the residual network.} %
    \label{tab:celeba-full}
	\resizebox{\textwidth}{!}{
    \setlength\tabcolsep{2.5pt}
    \sisetup{table-format=2.2,detect-weight=true}
    \begin{tabular}{ll|SSSSS|S|SSSSS|S} \toprule
    & & \multicolumn{6}{c|}{Convolutional ($d=512$)} & \multicolumn{6}{c}{ResNet18 ($d=2048$)} \\
    & & \multicolumn{5}{c}{task f1-scores (\%)~$\uparrow$} & &\multicolumn{5}{c}{task f1-scores (\%)~$\uparrow$} & \\
    & Method & {$\min_k$} & {$\med_k$} & {$\operatorname{avg}_k$} & {$\operatorname{std}_k$~$\downarrow$} & {$\max_k$} & {Hours} & {$\min_k$} & {$\med_k$} & {$\operatorname{avg}_k$} & {$\operatorname{std}_k$~$\downarrow$} & {$\max_k$} & {Hours} \\ 
    \midrule
    \multirow{5}{*}{\rotatebox[origin=c]{90}{Without $\mR_k$}} & Vanilla$\ssymbol{2}$ & 1.62 & 53.39 & 58.49 & 24.26 & 96.97 & 7.62 & 15.45 & 63.04 & 62.85 & 22.09 & 96.58 & 1.49 \\
    \cmidrule{2-14}
    & GradDrop & 2.63 & 52.32 & 57.33 & 25.27 & 96.72 & 8.53 & 13.31 & 64.37 & \maxf{63.95} & 20.93 & 96.59 & 1.60 \\
    & PCGrad & 2.69 & 54.60 & 56.87 & 25.75 & 97.04 & 34.05 & 13.61 & 62.45 & 62.74 & 21.60 & 96.64 & 5.75 \\
    & GradNorm & 2.17 & 52.98 & 56.91 & 24.72 & 96.84 & 20.93 & 17.42 & 62.49 & 62.62 & 21.93 & 96.55 & 3.61 \\
    & IMTL-G & 0.00 & 14.81 & 31.90 & 33.58 & 93.31 & 9.46 & 9.87 & 62.22 & 62.03 & 22.47 & 96.51 & 1.73 \\ \midrule
    \multirow{7}{*}{\rotatebox[origin=c]{90}{With $\mR_k$}} & Vanilla & 4.24 & 49.85 & 55.33 & 26.03 & 96.88 & 16.29 & 19.71 & 63.56 & 63.23 & 21.16 & 96.55 & 9.33 \\
    \cmidrule{2-14}
    & GradDrop & 3.18 & 50.07 & 54.43 & 27.21 & 96.80 & 17.20  & 12.33 & 62.40 & 62.74 & 21.74 & 96.65 & 9.41 \\
    & PCGrad & 1.44 & 53.05 & 54.72 & 27.61 & 96.90 & 41.79 & 14.71 & 63.65 & 62.61 & 22.22 & 96.59  & 13.72 \\
    & GradNorm & 2.08 & 52.53 & 56.71 & 24.57 & 96.96 & 30.02 & 9.05 & 60.20 & 60.78 & 22.31 & 96.38 & 11.36 \\
    & IMTL-G & 0.00 & 37.00 & 42.24 & 33.46 & 94.34 & 18.05 & 17.11 & 61.68 & 60.72 & 22.80 & 96.44 & 9.52 \\
    & \ours & 4.59 & 55.02 & 57.20 & 24.75 & 96.79 & 27.20 & 9.96 & 63.84 & 62.81 & 21.80 & 96.45 & 6.68 \\
    \bottomrule
    \end{tabular}
	}
\end{table}

\subsubsection{NYUv2}

Complementing the results shown in \Cref{sec:experiments}, we show in \Cref{tab:nyuv2-all} the results obtained combining \ours with all other existing methods (rows within \ours $+$), for gradient scaling methods we only apply the rotation part of \ours.
Results show that \ours helps improve/balance all other methods, which is specially true for those methods that heavily overlook some tasks. 
Specifically, MGDA-UB stops overlooking the semantic segmentation and depth estimation tasks, while PCGrad and GradDrop stop completely overlooking the surface normal loss. 
Note that we also show in \Cref{tab:nyuv2-all} the training times of each method, and \ours stays on par with non-extended methods in training time. 
As mentioned in \cref{app:setups}, due to cluster overload, some times were deceivingly high (specifically those baselines with $\mR_k$) as we had to run them on different machines, and were omitted to avoid confusion.

{
\setlength\tabcolsep{0.55em}
\sisetup{round-mode = places, round-precision = 2, detect-weight=true,detect-inline-weight=math}
\begin{table*}[t]
    \centering
    \caption{Results for different methods on the NYUv2 dataset with a SegNet model. \ours obtains the best performance in segmentation and depth tasks on all metrics, while significantly improving the results on normal surfaces with respect to the vanilla case.}
    \label{tab:nyuv2-all}
    \resizebox{\textwidth}{!}{
    \begin{tabular}{ll|S[round-precision=1, table-format=3.1]S[round-precision=1, table-format=2.1]S[round-precision=1, table-format=2.1]|SSSSSSSSS|c} \toprule
    & & \multicolumn{3}{c|}{\multirow{2}{*}{\shortstack{Relative\\improvements~$\uparrow$}}} & \multicolumn{2}{c}{\multirow{1}{*}{\shortstack{\textbf{S}egmentation~$\uparrow$}}} & \multicolumn{2}{c}{\multirow{1}{*}{\shortstack{\textbf{D}epth~$\downarrow$}}} & \multicolumn{5}{c|}{\textbf{N}ormal Surfaces} & Time~$\downarrow$ \\
     & & \multicolumn{3}{c|}{} & \multicolumn{2}{c}{} & \multicolumn{2}{c}{} & \multicolumn{2}{c}{Angle Dist.~$\downarrow$} & \multicolumn{3}{c|}{Within $t^{\circ}$~$\uparrow$} &  \\ 
    & Method & {${\Delta_S}$} & {${\Delta_D}$} & {${\Delta_N}$} & {\small mIoU} & \multicolumn{1}{c}{\small Pix Acc} & {\small Abs.} & \multicolumn{1}{c}{\small Rel.} & {\small Mean} & {\small Median} & {\small $11.25$} & {\small $22.5$} & \multicolumn{1}{c|}{\small $30$} & \SI{}{\hour} \\ \midrule
    & Single & 0.00 & 0.00 & \maxf{0.00} & 39.21 & 64.59 & 0.70 & 0.27 & \maxf{25.09} & \maxf{19.18} & \maxf{30.01} & \maxf{57.33} & \maxf{69.30} & 8.90 \\ \midrule \midrule
    \multirow{5}{*}{\rotatebox[origin=c]{90}{\ours $+$}} & GradDrop & \maxf{1.17} & 12.59 & \maxf{-7.51} & \maxf{40.26} & \maxf{65.63} & 0.63 & 0.24 & \maxf{26.33} & \maxf{21.08} & \maxf{26.47} & \maxf{53.38} & \maxf{66.05} & 3.94 \\
    & PCGrad & \maxf{-0.01} & \maxf{19.72} & \maxf{-8.32} & \maxf{39.08} & \maxf{64.68} & 0.54 & \maxf{0.21} & \maxf{26.41} & \maxf{21.29} & \maxf{26.13} & \maxf{52.99} & \maxf{65.72} & 3.89\\
    & MGDA-UB & \maxf{2.46} & \maxf{23.18} & \maxf{-8.11} & 39.32 & \maxf{65.48} & \maxf{0.54} & \maxf{0.21} & \maxf{26.43} & \maxf{21.22} & \maxf{26.16} & \maxf{53.16} & \maxf{66.07} & 3.85 \\
    & GradNorm & \maxf{1.06} & \maxf{21.38} & \maxf{-7.70} & 39.08 & \maxf{65.43} & \maxf{0.54} & \maxf{0.21} & \maxf{26.44} & \maxf{21.42} & \maxf{26.17} & \maxf{52.59} & \maxf{65.52} & 3.84 \\
    & IMTL-G & 1.65 & 21.19 & \maxf{-6.93} & 40.13 & 65.17 & 0.55 & \maxf{0.21} & \maxf{26.20} & \maxf{21.06} & \maxf{26.69} & \maxf{53.39} & \maxf{66.04} & 3.96 \\ \midrule
    \multirow{9}{*}{\rotatebox[origin=c]{90}{\shortstack{With $\mR_k$ $(m=1024)$}}} & Rotate Only & \maxf{3.28} & \maxf{20.49} & \maxf{-6.56} & \maxf{39.63} & \maxf{66.16} & \maxf{0.53} & \maxf{0.21} & \maxf{26.12} & \maxf{20.93} & \maxf{26.85} & \maxf{53.76} & \maxf{66.50} & 3.82 \\ %
    & Scale Only & \maxf{-0.27} & \maxf{20.01} & \maxf{-7.90} & \maxf{38.89} & \maxf{65.94} & 0.54 & 0.22 & \maxf{26.47} & \maxf{21.24} & \maxf{26.24} & \maxf{53.04} & \maxf{65.81} & 3.87 \\
    & \ours & \maxf{1.83} & \maxf{24.04} & \maxf{-6.11} & \maxf{39.32} & \maxf{66.07} & \maxf{0.53} & 0.21 & \maxf{26.01} & \maxf{20.80} & \maxf{27.18} & \maxf{54.02} & \maxf{66.53} & 3.83 \\ \cmidrule{2-15} %
     & Vanilla & -2.66 & 20.58 & -25.70 & 38.05 & 64.39 & 0.54 & 0.22 & 30.02 & 26.16 & 20.02 & 43.47 & 56.87 & 3.81 \\  \cmidrule{2-15}
    & GradDrop & -0.90 & 13.97 & -25.18 & 38.79 & 64.36 & 0.59 & 0.24 & 29.80 & 25.81 & 19.88 & 44.08 & 57.54 & 4.01 \\
    & PCGrad & -2.67 & 20.47 & -26.31 & 37.15 & 63.44 & 0.55 & \maxf{0.22} & 30.06 & 26.18 & 19.58 & 43.51 & 56.87 & 3.89 \\
    & MGDA-UB & -31.23 & -0.65 & \maxf{0.59} & 21.60 & 51.60 & 0.77 & 0.29 & \maxf{24.74} & \maxf{18.90} & \maxf{30.32} & \maxf{57.95} & \maxf{69.88} & 3.85 \\
    & GradNorm & -0.55 & 19.50 & \maxf{-10.45} & 37.22 & 63.61 & 0.54 & 0.22 & \maxf{26.68} & \maxf{21.67} & \maxf{25.95} & \maxf{52.16} & \maxf{64.95} & 3.85 \\
    & IMTL-G & -0.32 & \maxf{17.56} & \maxf{-7.46} & 38.38 & \maxf{64.66} & \maxf{0.54} & 0.22 & \maxf{26.38} & \maxf{21.35} & \maxf{26.56} & \maxf{52.84} & \maxf{65.69} & 3.99 \\ \cmidrule{1-15}
    \multirow{6}{*}{\rotatebox[origin=c]{90}{Without $\mR_k$}} & Vanilla$\ssymbol{2}$ & -0.94 & 16.77 & -25.03 & 37.11 & 63.98 & 0.56 & 0.22 & 29.93 & 25.89 & 20.34 & 43.92 & 57.39 & 3.46 \\
    & GradDrop & -0.10 & 15.71 & -26.99 & 37.51 & 63.62 & 0.59 & 0.23 & 30.15 & 26.33 & 19.32 & 43.15 & 56.59 & 3.55 \\
    & PCGrad & -0.51 & 19.97 & -24.63 & 38.51 & 63.95 & 0.55 & 0.22 & 29.79 & 25.77 & 20.61 & 44.22 & 57.64 & 3.51 \\
    & MGDA-UB & -32.19 & -8.22 & \maxf{1.50} & 20.75 & 51.44 & 0.73 & 0.28 & \maxf{24.70} & \maxf{18.92} & \maxf{30.57} & \maxf{57.95} & \maxf{69.99} & 3.52 \\
    & GradNorm & 2.18 & \maxf{20.60} & \maxf{-10.23} & 39.29 & \maxf{64.80} & \maxf{0.53} & \maxf{0.22} & \maxf{26.77} & \maxf{21.88} & \maxf{25.39} & \maxf{51.78} & \maxf{64.76} & 3.50 \\
    & IMTL-G & \maxf{1.92} & \maxf{21.35} & \maxf{-6.71} & \maxf{39.94} & \maxf{65.96} & \maxf{0.55} & \maxf{0.21} & \maxf{26.23} & \maxf{21.14} & \maxf{26.77} & \maxf{53.25} & \maxf{66.22} & 3.61 \\ \bottomrule
    \end{tabular}
	}
    \vspace{-10pt}
\end{table*}
}

        \addtocontents{toc}{\endgroup}
    \end{appendices}
\end{document}